\documentclass{article} 
\usepackage[final]{colm2026_conference}

\usepackage{microtype}
\usepackage{hyperref}
\usepackage{url}
\usepackage{booktabs}
\usepackage{amsmath}
\usepackage{amsfonts}
\usepackage{amssymb}
\usepackage{bm}
\usepackage{cleveref}
\usepackage{algorithm}
\usepackage{algorithmic}
\usepackage[table]{xcolor}   
\usepackage{stmaryrd}
\usepackage{tikz}
\usepackage{wrapfig}
\usetikzlibrary{shadings}
\usetikzlibrary{arrows.meta}
\usepackage{twemojis}

\usepackage{booktabs}   
\usepackage{multirow}   
\usepackage{graphicx}   

\usepackage{amsthm}
\newtheorem{definition}{Definition}
\crefname{definition}{Definition}{Definitions}
\newtheorem{remark}{Remark}
\crefname{remark}{Remark}{Remarks}
\newtheorem{lemma}{Lemma}
\newtheorem{assumption}{Assumption}
\crefname{assumption}{Assumption}{Assumptions}
\newtheorem{corollary}{Corollary}
\usepackage{thmtools}
\declaretheorem[name=Proposition]{proposition}

\usepackage{lineno}

\definecolor{darkblue}{rgb}{0, 0, 0.5}
\hypersetup{colorlinks=true, citecolor=darkblue, linkcolor=darkblue, urlcolor=darkblue}

\title{A Tale of Two Temperatures: Simple, Efficient, and Diverse Sampling from Diffusion Language Models}

\author{%
  Theo X. Olausson$^{1}$\thanks{Equal contribution. Correspondence to \texttt{m.jazbec@uva.nl}, \texttt{theoxo@csail.mit.edu}. Code available at: \url{https://github.com/metodj/2-temps-dLLM}.} \quad Metod Jazbec$^{2*}$ \quad Xi Wang$^{3}$ \quad Armando Solar-Lezama$^{1}$ \\
  \textbf{Christian A. Naesseth}$^{2}$ \quad \textbf{Stephan Mandt}$^{4}$ \quad \textbf{Eric Nalisnick}$^{3}$ \\
  $^1$Massachusetts Institute of Technology \quad $^2$UvA-Bosch Delta Lab, University of Amsterdam \\ 
  $^3$Johns Hopkins University \quad $^4$University of California, Irvine
}

\newcommand{\mask}{\texttt{[MASK]}}
\newcommand{\indicator}[1]{\bm{1}\left( #1 \right)}
\newcommand{\tpos}{T_\text{pos}}

\begin{document}

\ifcolmsubmission
\linenumbers
\fi

\maketitle

\begin{abstract}
Much work has been done on designing fast and accurate sampling for diffusion language models (dLLMs).
However, these efforts have largely focused on the tradeoff between
speed and quality of \emph{individual} samples; how to additionally ensure
\emph{diversity} across samples remains less well understood.
In this work, we show that diversity can be increased by using
softened, tempered versions of familiar confidence-based remasking heuristics,
retaining their computational benefits and offering simple
implementations.
We motivate this approach by introducing an idealized formal model of \emph{fork tokens} and studying the impact of remasking on the expected entropy at the forks. Empirically, the proposed tempered heuristics close the exploration gap (pass@$k$) between existing confidence-based and autoregressive sampling, hence outperforming both when controlling for cost (pass@NFE). We further study how the increase in diversity translates to downstream post-training and test-time compute scaling.
Overall, our findings demonstrate that simple, efficient, and diverse
sampling from dLLMs is possible.
\end{abstract}

\section{Introduction}

Diffusion language models (dLLMs) have recently seen significant uptake,
offering native bidirectional context and the ability to generate
multiple tokens in parallel
(\citealt{nie2025large,ye2025dream,yang2025mmada,bethune2026designspacetrimodalmasked}; \emph{inter alia}).
Unlike in the autoregressive (AR) setting, dLLMs require a \emph{remasking} strategy that determines the order in which tokens are sampled. Because this choice significantly affects sample quality, much prior work has focused on optimizing the speed--quality tradeoff for individual samples
(\citealt{chang2022maskgit,nie2025large,kim2025train,ben-hamuAcceleratedSamplingMasked2025,jazbec2025learningunmaskingpoliciesdiffusion}; \emph{inter alia}).


While this line of work has helped make dLLMs more competitive with AR LLMs, efficiently obtaining a \emph{single} high-quality sample is not always sufficient: post-training, test-time compute scaling, and A/B testing all require \emph{diversity} across samples in order to be effective. Surprisingly, \citet{ni2026flexibility} found that confidence-based heuristics yielding high single-sample accuracy lead to worse diversity than left-to-right autoregressive generation. By delaying uncertain tokens, such heuristics avoid genuine branching points in the reasoning process. Autoregressive generation does not suffer from this problem, but comes at a significant cost as it forgoes parallel generation of multiple tokens altogether.

In this work, we argue that the diversity problem can be addressed with a simple fix, without abandoning parallel generation or adopting complex sampling schemes.
Our key observation is that it is the greedy nature of existing heuristics, not parallel generation in itself, that limits diversity.
By replacing them with softened, tempered approximations thereof---for example, sampling positions proportional to their tempered confidence instead of selecting them greedily based on the highest confidence---we demonstrate that diversity can be recovered while preserving speed and simplicity.

In summary, our contributions are as follows:
\begin{itemize}
    \item We introduce \emph{Tempered Low-Confidence} (TLC) and
    \emph{Tempered Confidence Thresholding} (TCT) remasking,
    stochastic relaxations of two widely used remasking
    heuristics \citep{nie2025large, wu2025fastdllmtrainingfreeaccelerationdiffusion} controlled by a \emph{position temperature} $\tpos$ (\Cref{sec:methods_temp}).
    This provides a knob complementary to the standard
    token-level temperature $T_\text{token}$:
    $T_\text{token}$ governs the entropy of what token is predicted at each
    position, while $\tpos$ governs how stochastic the order in which positions are committed
    to is.
    
    \item We give an operational definition of \emph{fork tokens}, high-entropy positions that prior work has empirically found to be important for effective post-training, and show that, under an idealized
    model, increasing $\tpos$ provably increases the expected entropy at these forks for TLC (\Cref{sec:methods_forks}).

    \item We empirically validate that TLC recovers pass@$k$ scaling
    comparable to AR rollouts across multiple benchmarks
    and models, confirming that tempering the remasking strategy is
    sufficient to restore diversity (\Cref{sec:exp_passk}, \Cref{fig:llada_tcl_all_ds}).
    TCT goes further by retaining the adaptive speed of confidence
    thresholding; when the cost of each rollout is accounted for,
    TCT often scales better than either untempered heuristics or AR
    generation (\Cref{sec:exp_passk}, \Cref{fig:llada_tct_all_ds}).

    \item We study how pass@$k$ improvements translate downstream:
    when used for GRPO rollout generation in dLLM post-training \citep{zhao2025d1} at matched computational budgets,
    TCT yields stronger policies than AR generation (\Cref{sec:exp_rl}).
    In the context of test-time compute, we find that while the increased diversity of TCT can
    overwhelm simple strategies such as majority voting, this gap is
    largely closed when using outcome reward models for answer
    selection (\Cref{sec:ttc}).
\end{itemize}

\section{Background}

We use the notation $[L]$ to represent the set $\{1, \ldots, L\}$. We denote a sequence of tokens by $\bm{x} = (x^1, \ldots, x^L) \in \mathcal{V}^L$, where $\mathcal{V}:= [V]$ is the vocabulary and $L$ is the sequence length.

\subsection{From masked diffusion models to diffusion language models}
Masked Diffusion Models (MDMs) are a form of absorbing discrete 
diffusion \citep{jascha2015diffusion,austin2021} in which the target 
(noise) distribution places all its mass on a single token, \mask,
and the forward process gradually corrupts each token towards the 
absorbing mask state by transitioning tokens independently through 
a Markov chain.
Formally, this can be characterized compactly as follows:
\begin{gather}
\bm{x}_0 \sim p_\text{data} \qquad \bm{x}_1 \sim 
\delta_\mask^{\otimes L} \\\label{eq:fwd}
\bm{x}_t \mid \bm{x}_0 \sim p_t(\bm{x}_t\mid \bm{x}_0)
\triangleq \prod_{k=1}^L \alpha(t)\indicator{x_t^k = \mask} 
+ (1-\alpha(t))\indicator{x_t^k = x_0^k}
\end{gather}
where $0 \leq t \leq 1$ denotes 
the time, $\alpha : t \rightarrow [0, 1]$ is a
monotonically 
non-decreasing function in $t$ with endpoints $\alpha(0) = 0, \alpha(1) = 1$, and $\delta^{\otimes L}_\mask$ is a product of $L$ Dirac measures placing all of their mass on $\mask$.
The generative (i.e., reverse) process can be modeled either by 
learning to predict the time-dependent reverse transition ratios 
\citep{campbell2021,mengConcreteScoreMatching2022,lou24sedd}, 
or by training a model $q_\theta$ to directly approximate the 
conditional distribution $p_{0\mid t}$.
The latter strategy has typically
been the choice of language-modeling practitioners,
as later work \citep{sahoo2024simple,shi2024simplified,ou2025}
showed that its evidence lower bound (ELBO) simplifies to:
\begin{equation}
\label{eq:training}
    -\mathcal L (\theta) \triangleq 
    \mathbb{E}_{t \sim U[0, 1], \bm{x}_0 \sim p_\text{data}, 
    \bm{x}_t \sim p_t(\cdot \mid \bm{x}_0)} \left[ 
    \frac{\alpha'(t)}{\alpha(t)} \sum_{k=1}^L \indicator{x_t^k = \mask}
    \log q_\theta^k(x_0^k \mid \bm{x}_t) \right] \ 
\end{equation}
where $q_\theta^k(\cdot| \bm{x}_t)$ is the model's predicted \emph{marginal} distribution over possible tokens at the $k$-th position.
This formulation eliminates explicit conditioning on $t$,
connecting MDM training to masked language modeling \citep{devlin-etal-2019-bert}
and any-order autoregressive models \citep{hoogeboom2022ardm}.
Combined with empirically validated scaling laws for MDMs \citep{nie2025scaling,bethune2026designspacetrimodalmasked}, 
this has enabled the open-source community to train multi-billion parameter MDMs on textual data, which we refer to as \emph{diffusion (large) language models} (dLLMs).

\subsection{Remasking strategies and sampling from dLLMs}
\label{sec:bg-remasking}

Early work on MDMs obtained samples by discretizing the reverse process and simulating the Markov chain \citep{campbell2021,austin2021,lou24sedd}.
As observed by \citet{zhengMaskedDiffusionModels2025}, this is potentially wasteful as each individual token transitions exactly once in \Cref{eq:fwd}.
They instead proposed a \emph{first-hit} sampling scheme, which when the time-agnostic objective of \Cref{eq:training} is used
simplifies to
what \citet{nie2025large} call \emph{random remasking}. Concretely, denoting the current generation as $\bm x_t \in \mathcal{V}^L$ with still-masked positions $\mathcal M_t := \{k \in  [L] \mid x_t^k = \mask\}$, the next generation $\bm x_{t'}$ is obtained via
\begin{align*}
    x_{t'}^k := \begin{cases}
        x \sim q_{\theta}^k(\cdot \mid \bm x_t \: ; T_{\text{token}}), & \text{if } k \in \mathcal U_t, \\
        x_{t}^k\;, & \text{otherwise.}
    \end{cases}
\end{align*}
where the \emph{unmasking set} $\mathcal{U}_t$ of $K \geq 1$ positions is sampled uniformly at random (without replacement) from $\mathcal M_t$.
As in the autoregressive setting, $T_\text{token}$ tempers the final softmax of $q_\theta^k$: intuitively, higher temperatures yield more diverse (but possibly lower quality) generations.
This process continues until reaching a time where $\mathcal M_t  = \emptyset$.

Later work revealed that in practice, test-time performance can 
often be improved by using \emph{heuristic} remasking strategies 
that select $\mathcal{U}_t$ based on distributional information 
about the model's predictions. Two particularly popular choices use the model's per-position confidence 
$c_{t}^k \triangleq \max_{v} q_\theta^k(v \mid \bm x_t)$:\footnote{For non-greedy sampling we define the confidence as $c_{t}^k \triangleq q_\theta^k(x \mid \bm x_t), \: x \sim q_{\theta}^k(\cdot \mid \bm x_t \: ; T_{\text{token}})$.}
\emph{low-confidence} 
(LC; \citealt{chang2022maskgit,nie2025large}) remasking unmasks the $K$ most 
confident positions, while \emph{confidence thresholding} 
(CT; \citealt{wu2025fastdllmtrainingfreeaccelerationdiffusion}) unmasks 
all positions exceeding a fixed threshold 
$\lambda \in [0, 1]$:
\begin{align}
\label{eq:lc}
\text{LC:} \quad \mathcal U_{t}^{K} &:= 
\underset{I 
\subseteq \mathcal M_t,\ |I| = K}{\arg\max}\ 
\sum_{k \in I} c_{t}^k \\
\label{eq:ct}
\text{CT:} \quad \mathcal U_t^{\lambda} &:= \{ k \in 
\mathcal M_t \mid c_t^k > \lambda \}
\end{align}
The key distinction is that low-confidence remasking unmasks a fixed number of tokens per step ($K$), while confidence thresholding is \emph{adaptive}: $|\mathcal{U}^\lambda_t|$ varies with the model's confidence, allowing sampling to proceed quickly when the model is certain and slowly when it is not.

\subsection{Sample diversity and the role of remasking}
\label{sec:bg-diversity}

The remasking strategies discussed above are typically compared in
terms of accuracy at a given computational
budget under greedy, $T_{\text{token}}=0$, decoding
\citep{nie2025large,wu2025fastdllmtrainingfreeaccelerationdiffusion,ben-hamuAcceleratedSamplingMasked2025}.
However, many practical applications require diverse \emph{sets}
of samples; in this case the more important metric becomes pass@$k$,
the probability that at least one of $k$ independent
samples solves a given problem \citep{kulal2019spoc,chen2021evaluatinglargelanguagemodels}.
\citet{ni2026flexibility} investigated pass@$k$ scaling for dLLMs
and found that low-confidence remasking, despite its strong
greedy performance, exhibits notably flatter pass@$k$ curves
than autoregressive (left-to-right) decoding.
They attribute this to a phenomenon they dub \emph{entropy degradation},
where 
low-confidence systematically defers high-entropy
branching positions until surrounding context has already been established,
collapsing the entropy that would otherwise enable
qualitatively diverse samples.
Autoregressive decoding, \citet{ni2026flexibility} argue, sidesteps this by resolving tokens in a
fixed positional order agnostic to confidence, forcing the model to
commit to branching points as they arise.

\section{Methods}
\label{sec:methods}

\begin{figure}
    \centering
    \includegraphics[width=\linewidth]{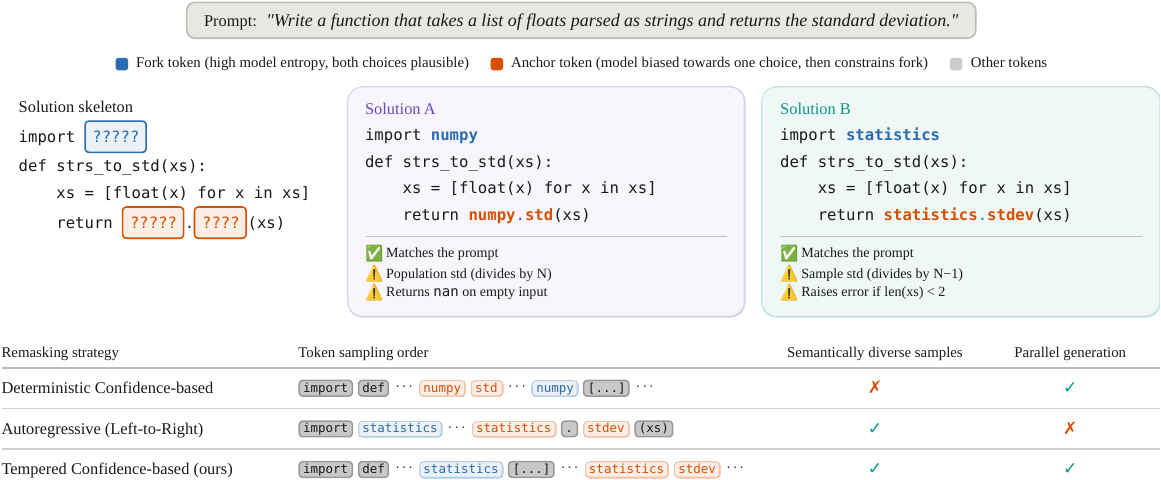}
    \caption{An example illustrating our model of entropy degradation. The \textcolor[HTML]{2B6CB0}{\textbf{fork token}} (\Cref{def:fork-token})
    has high marginal uncertainty under the dLLM,
    indicating that both the \texttt{numpy} and \texttt{statistics} solutions are plausible. However, the dLLM is simultaneously overly confident in its prediction at the \textcolor[HTML]{D94F00}{\textbf{anchor tokens}}
    in the function body, biasing low-confidence remasking towards one solution as it reveals those tokens first.
    Tempered remasking allows the fork to still be revealed first, increasing semantic diversity without sacrificing parallelism.
    Note that while both paths lead to generations that match the prompt (\twemoji{2705}), they have subtly different semantics (\twemoji{warning}); sampling both is important to achieve good coverage.}
    \label{fig:overview}
\end{figure}

We now introduce our approach for generating diverse and efficient samples with dLLMs.
We first study \citet{ni2026flexibility}'s empirical findings through the lens of a formal model of \emph{fork tokens} and their relationship to semantic sample diversity (\Cref{sec:methods_forks}).
Motivated by our observations, we then propose
to temper confidence-based heuristics via a new position temperature $T_{\text{pos}}$, and discuss two concrete instantiations of this strategy (\Cref{sec:methods_temp}).
\Cref{fig:overview} shows an overview of our modeling framework and methodology.

\subsection{Fork tokens and the anchor-fork degeneracy}
\label{sec:methods_forks}

\citet{ni2026flexibility}'s empirical study of entropy degradation is grounded
in the observational definition of \emph{forking tokens} given by
prior work (see \Cref{app:related_work}).
These works identify forks as positions in the
sequence that have exceptionally high entropy,
and observe that maintaining diversity at these positions translates to effective post-training (for autoregressive models).

To allow us to reason more precisely about this phenomenon,
we develop a formal model in Appendix~\ref{appendix:proofs} (see Appendix~\ref{app:assump1_empirical} for empirical validation).
Key to our analysis is an \emph{operational} definition of fork tokens as positions in the sequence whose value
sharply bounds the set of possible semantic outcomes under the data distribution (\Cref{def:fork-token}), allowing us
to formally connect their uncertainty to the semantic sample diversity (Equation~\ref{eq:sandwich}).
To connect this machinery to \citet{ni2026flexibility}'s observations of
the \emph{sampling distribution} induced by pairing $q_\theta$ with low-confidence or autoregressive remasking,
we build on their intuition that the root cause of
entropy degradation is that contextual tokens surrounding the fork
tend to dominate it in confidence, and that revealing them collapses the fork's entropy before it is ever sampled.
We capture this intuition in what we call the \emph{anchor-fork degeneracy} model
(Assumption~\ref{assumption:anchor-fork}),
where the sampling distribution's uncertainty about the fork token
is determined by the set of high-confidence \emph{anchor tokens} that have been revealed so far. This phenomenon is visualized in the upper half of \Cref{fig:overview}.

This simple model recovers all three of
\citet{ni2026flexibility}'s empirical observations:
(i)~LC yields minimal average fork entropy;
(ii)~adjusting $T_\text{token}$ does not help with preserving entropy at forks, since it does not
change the unmasking \emph{order}
(Proposition~\ref{prop:token-temp});
and (iii)~left-to-right generation yields strictly higher average fork
entropy than LC whenever at least one anchor follows the fork in
positional order (Proposition~\ref{prop:ar-diversity}).
Additionally, however, it also yields a new prediction:
autoregressive generation is not special, and \emph{any} strategy
that increases the probability of resolving a fork before its
anchors will retain higher fork entropy than LC.
This suggests that semantic diversity can be improved
without sacrificing the efficiency of parallel generation.

\subsection{Tempered remasking heuristics}
\label{sec:methods_temp}

The formal analysis of Appendix~\ref{appendix:proofs} suggests a simple fix for confidence-based heuristics: rather than
deterministically selecting positions to unmask,
\emph{sample} them in proportion to a softened distribution determined by their confidences.
This distribution can then be controlled through
a \emph{position temperature} $\tpos$ that
provides a complementary knob to $T_\text{token}$:
$T_\text{token}$ governs how greedily we determine \emph{what} is sampled at each position,
while $\tpos$ governs how greedily we determine the \emph{order} in which to do so.

We now describe concrete instantiations corresponding to each heuristic from \Cref{sec:bg-remasking}.

\textbf{Tempered low-confidence (TLC) remasking (\Cref{alg:tlc}).}
Rather than deterministically selecting $\mathcal U_t^K$ as in
\Cref{eq:lc},
TLC samples $K$ positions (without replacement) from a tempered
distribution induced by the confidences:
\begin{align*}
\mathcal U_t^{K, \tpos}
\stackrel{\text{w/o repl.}}{\sim} \text{Cat}(\tilde{\bm c}_t),
\qquad
\tilde c_t^k \propto (c_t^k)^{1/\tpos} \cdot \indicator{k \in \mathcal M_t}
\end{align*}
where $\tpos > 0$ is a free parameter.
As $\tpos \downarrow 0$, TLC recovers deterministic
low-confidence remasking;
as $\tpos \uparrow \infty$, it approaches uniform (random)
remasking.

In Appendix~\ref{appendix:proofs} we prove that TLC resolves the
anchor-fork degeneracy: under the idealized model, sufficiently
large increases of $\tpos$ provably increase the expected
fork-token entropy
(Proposition~\ref{prop:fork-entropy}); combined with an assumption
that $q_\theta$ captures the data distribution well enough to preserve the fork structure (\Cref{assumption:model-quality}), 
this yields a direct increase in semantic diversity, and thus intuitively pass@$k$, whenever the fork-entropy gain is large enough (Corollary~\ref{cor:semantic-entropy}).%
\footnote{In practice, excessively increasing $\tpos$ risks revealing positions with less context than is needed to yield
coherent generations (same as random unmasking), so for smaller values of $k$ the optimal $\tpos$ needs to be chosen to balance the
diversity gain against this loss of quality
(see \Cref{fig:t_pos_ablate}).
}

While TLC offers a mechanism to recover sample diversity,
it does not yield the \emph{adaptive} speed-ups that make confidence-based heuristics particularly attractive.
This motivates applying the same tempering idea to confidence thresholding.
 
\textbf{Tempered confidence thresholding (TCT) remasking (\Cref{alg:tct}).}
TCT replaces the hard threshold in \Cref{eq:ct} with stochastic inclusion:
\begin{align*}
\mathcal U_t^{\lambda, \tpos}
= \{k \in \mathcal M_t \mid b_t^k = 1\},
\qquad
b_t^k \sim \text{Ber}\!\left(
\sigma\!\left( (c_t^k - \lambda) / \tpos \right)
\right)
\end{align*}
where $\sigma(x) = 1/(1 + e^{-x})$ is the sigmoid function.
This recovers the hard threshold at $\lambda$ as $\tpos \downarrow 0$,
while yielding increasingly stochastic (and $\lambda$-agnostic) behavior as $\tpos$ increases.
Intuitively, at moderate temperatures, $\lambda$ controls the speed of generation while $\tpos$ controls the degree of diversity:
under hard thresholding, a fork token whose confidence falls below $\lambda$ is deterministically deferred;
under TCT, the sigmoid gives it nonzero probability of being unmasked at every step, breaking the systematic deferral.
TCT thus retains the adaptive efficiency of confidence thresholding while introducing the same diversity benefits as TLC.

\section{Experiments}
\label{sec:experiments}

We begin our experiments by verifying that our proposed tempered samplers can close the gap to autoregressive sampling in terms of pass@$k$ and also outperform it when the cost of generation is taken into account (\Cref{sec:exp_passk}). Next, we investigate whether the improvements in diversity translate to improved downstream performance for test-time compute (\Cref{sec:ttc}). Finally, we examine the implications of tempering when generating rollouts during RL post-training of dLLMs (\Cref{sec:exp_rl}).

\subsection{Tempered heuristics recover pass@$k$ scaling}
\label{sec:exp_passk}

We start by investigating the core hypothesis that tempering the heuristics can recover pass@$k$ scaling comparable to autoregressive sampling \citep{ni2026flexibility}.

\textbf{Experimental setting.}\footnote{Our code is publicly available at: \url{https://github.com/metodj/2-temps-dLLM}.}
We evaluate our proposed remasking strategies on two open-source dLLMs: LLaDA-8B-Instruct \citep{nie2025large} and Dream-v0-Instruct-7B \citep{ye2025dream}. We evaluate on four standard benchmarks: GSM8K \citep{cobbe2021gsm8k}, MATH-500 \citep{hendrycksmath2021}, HumanEval \citep{chen2021evaluatinglargelanguagemodels}, and MBPP \citep{austin2021programsynthesislargelanguage}.
For GSM8K we evaluate on a subset of $N_{\text{test}}=300$ test samples to reduce computational cost. Following prior work \citep{zhao2025d1}, we use $T_\text{token} = 0.8$ as well as semi-AR/block-wise decoding with blocks of 32 tokens \citep{arriola2025block,nie2025large}, at a sequence length of $L = 256$. 
Due to cost, and to avoid overfitting the position temperature to any one particular setting, we fix $\tpos = 1$ for TLC and $\tpos = 0.1$ for TCT.

We report pass@$k$, the standard measure of whether at least one of $k$ samples solves a given problem \citep{kulal2019spoc,chen2021evaluatinglargelanguagemodels}.
When evaluating the adaptive remasking functions,
we also report \emph{pass@NFE}---the pass rate as a function of the cumulative NFEs required.
This metric was also used in recent work by \citet{shen2026improvingdiffusionlanguagemodel},
and avoids pass@$k$'s bias towards methods that use more compute per sample \citep{olausson2024repair}.
%
\begin{figure*}[t]
    \centering
    \includegraphics[width=\textwidth]{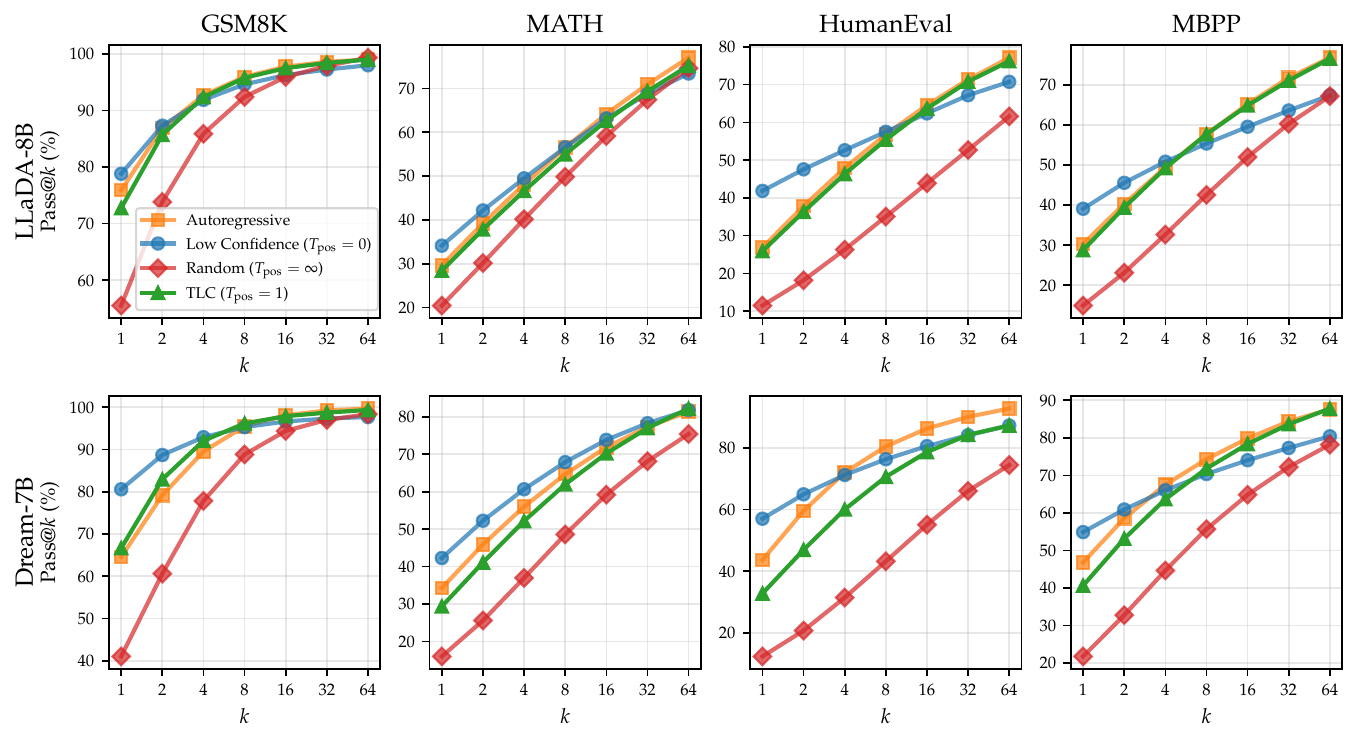}
    \caption{Pass@$k$ for TLC ($\tpos = 1$) on LLaDA-8B-Instruct and Dream-7B-Instruct across four benchmarks. TLC closely tracks the AR baseline, recovering almost all of the diversity lost by low-confidence remasking.}
    \vspace{-15pt}
    \label{fig:llada_tcl_all_ds}
\end{figure*}

\textbf{TLC recovers diversity at high sample counts.}
\Cref{fig:llada_tcl_all_ds} shows pass@$k$ for TLC across all four benchmarks, as well as for autoregressive, low-confidence and random baselines.
We first note that our results echo the findings of \citet{ni2026flexibility},
as the low-confidence baseline underperforms autoregressive generation for both models, on all four datasets.
In contrast, despite the lack of task- or model-specific tuning of $\tpos$,
TLC closely tracks the AR baseline in all settings other than for Dream on HumanEval.
These findings confirm that tempering the remasking heuristic can restore rollout diversity. Moreover, while the random baseline catches up to TLC and AR on math datasets, it exhibits a large gap on coding datasets (HumanEval, MBPP), indicating that too much randomness in sampling unmasking positions (i.e., $T_{\text{pos}}=\infty$) can hurt performance. 

In \Cref{fig:t_token_ablate}, we additionally study how robust these findings are across different token temperatures, finding TLC to be the least dependent on the exact choice of $T_{\text{token}}$. Concretely, performance under AR sampling degrades severely when using $T_{\text{token}} > 1$, while low-confidence experiences a stronger drop in performance for $T_{\text{token}} < 1$ compared to the proposed TLC. Importantly, TLC exhibits larger performance improvements when increasing $k$ compared to low-confidence sampling across all considered $T_{\text{token}}$, confirming the benefits of additional randomness (via $\tpos > 0$) on sampling diversity.

In \Cref{fig:tlc_mbpp_sdar_llada2}, we further investigate whether the amount of diffusion training received by the underlying dLLM influences the effectiveness of TLC. To this end, we repeat our TLC pass@$k$ experiments on LLaDA-2-mini \citep{bie2025llada2} and SDAR \citep{cheng2026sdar}, both of which were initialized from autoregressive LLMs and subsequently converted into dLLMs using approximately 200B and 50B training tokens, respectively. In contrast, LLaDA was trained as a diffusion model from scratch on 2.3T tokens. As shown in \Cref{fig:tlc_mbpp_sdar_llada2}, unlike on LLaDA, TLC does not fully close the gap to autoregressive sampling on either LLaDA-2-mini or SDAR. We hypothesize that this is because TLC shifts the sampling behavior toward random sampling, such that its effectiveness at large $k$ depends on random sampling itself retaining reasonable accuracy. Indeed, random sampling appears to perform better for models exposed to more extensive diffusion training: on LLaDA, it fully catches up with low-confidence sampling by $k=64$, whereas on SDAR it remains substantially behind. Taken together, these results suggest that TLC is most effective when the underlying model has received sufficient diffusion training. 

\begin{figure}[t]
    \centering
    \includegraphics[width=\textwidth]{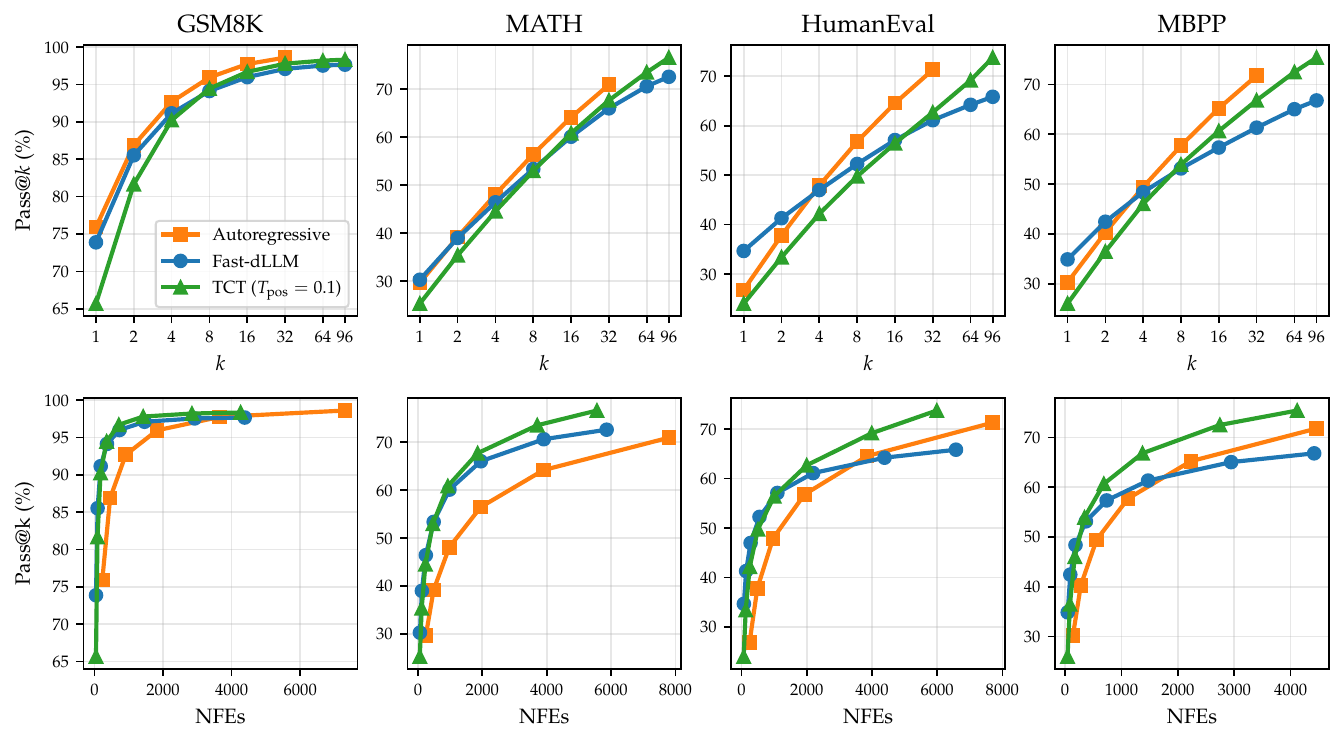}
    \caption{Pass@$k$ (top) and pass@NFE (bottom) for TCT on LLaDA-8B-Instruct. Both rows plot the same data; only the metric changes (pass@$k$ vs pass@NFE). TCT slightly underperforms AR in pass@$k$, but substantially outperforms it when the cost of each sample is taken into account. TCT also outperforms deterministic confidence-thresholding (Fast-dLLM; \citealt{wu2025fastdllmtrainingfreeaccelerationdiffusion}) thanks to its additional source of randomness via $T_\text{pos}$.}
    \vspace{-15pt}
    \label{fig:llada_tct_all_ds}
\end{figure}

\textbf{TCT: lower pass@$k$, but higher pass@NFE.}
We depict TCT sampling results for LLaDA in \Cref{fig:llada_tct_all_ds} (see \Cref{fig:dream_tct_all_ds} for Dream and \Cref{fig:llada2_tct_all_ds} for LLaDA-2-mini \citep{bie2025llada2} results). Looking only at pass@$k$ (top row), TCT ($\tpos = 0.1$, $\lambda = 0.6$) slightly underperforms the AR baseline on all benchmarks, although it still beats out the deterministic confidence-thresholding ($\lambda = 0.6$; \citealt{wu2025fastdllmtrainingfreeaccelerationdiffusion}). However, when cost is taken into account (bottom row), TCT starts to shine, substantially outperforming both baselines. As a concrete example, on the MATH dataset, TCT requires around $4000$ NFEs to achieve performance above $70\%$: a $2\text{x}$ speed-up compared to AR which requires around $8000$ NFEs.\footnote{When measuring the NFEs for AR sampling, we count the number of tokens generated before $\text{\textlangle eos\textrangle}$. For TCT/CT, we count NFEs until all $L$ positions are unmasked.} Notably, deterministic confidence-thresholding alone is not sufficient to outperform autoregressive rollouts at higher NFE values, as it underperforms AR on HumanEval and MBPP; only TCT consistently comes out on top, highlighting the importance of tempering confidence-based heuristics to preserve diversity.

\subsection{TCT for test-time compute}
\label{sec:ttc}

We next turn our attention to test-time compute to check whether the observed improved diversity of our proposed TCT sampler translates to  better downstream performance.

\textbf{Experimental setting.} We focus on evaluating LLaDA-8B-Instruct on GSM8K (same subset of $N_{\text{test}}=300$ samples) and MATH-500.
We study these datasets as they require the use of heuristic answer selection strategies, unlike in coding where pass@$k$ already directly captures scaling with respect to a fixed test suite.
We consider two widely used such strategies: (i) \emph{self-consistency}, where the most common answer is selected \citep{wang2023selfconsistencyimproveschainthought}, and (ii) \emph{ORM}, where an (outcome) reward model is used to score each generation, and the one with the highest score is returned. For the ORM, we use AceMath-7B \citep{acemath2024}.

\textbf{Results.} We report Best@NFE results in \Cref{fig:ttc}. Focusing first on the self-consistency results, we observe that while TCT outperforms the AR baseline at lower NFE counts, it plateaus sooner and thus achieves a lower top performance. As depicted in \Cref{fig:entropy-vs-k}, we attribute this to TCT producing less peaked, higher-entropy predictive distributions when aggregating over $k$ samples (see \Cref{fig:ans_dist_gsm8k} for concrete examples of such predictive distributions). However, when the ORM is used to select the sample, the performance trends from pass@NFE in the previous section are largely recovered, and both TCT and confidence-thresholding outperform the AR baseline in terms of Best@NFE.

\subsection{TCT for GRPO rollouts}
\label{sec:exp_rl}

\begin{figure}
    \vspace{-15pt}
    \centering
    \includegraphics[width=\linewidth]{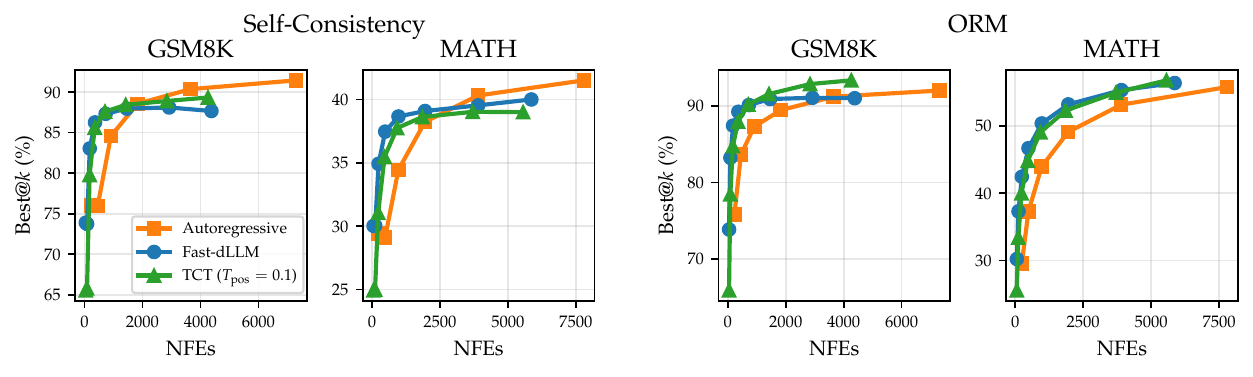}
    \caption{Test-time scaling for LLaDA-8B-Instruct in terms of Best@NFE when using either self-consistency or an ORM for final answer selection. TCT performs worse compared to AR at higher NFEs when using self-consistency; however, when using an ORM, TCT improves the Pareto frontier compared to autoregressive rollouts. See \Cref{fig:ttc_bestk_bestnfe} for Best@$k$ results.}
    \vspace{-15pt}
    \label{fig:ttc}
\end{figure}

Recall that our interest in pass@$k$ (and pass@NFE) stems from it being a useful indicator of the performance gains that can be obtained through policy optimization post-training \citep{yue2025limit-of-rlvr}.
In this section, we thus investigate whether the pass@NFE improvements of TCT translate to better post-training results, when controlling for training cost.

\textbf{Experimental setting.}
We focus on post-training LLaDA-8B-Instruct \citep{nie2025large} on mathematical datasets (GSM8K, MATH-500). We adopt the d1 reinforcement learning framework \citep{zhao2025d1}, as it is widely used for post-training dLLMs and provides efficient (albeit biased) policy likelihood estimation requiring only a single NFE. For generating GRPO rollouts, we use either TCT (with $\tpos = 0.1$, $\lambda = 0.6$), deterministic confidence-thresholding (with $\lambda = 0.6$), or autoregressive rollouts. Since the considered samplers differ in generation speed, we ensure a fair comparison by training for a fixed duration rather than a fixed number of training steps or epochs. Specifically, all models are trained for 72 hours on a single node with 4×A100 GPUs. After training, we evaluate the resulting RL-trained models under greedy decoding ($T_{\text{token}} = 0$) with confidence-thresholding, varying the threshold $\lambda \in \{0.6, 0.8, 1.0\}$ to obtain accuracy-NFEs Pareto frontiers for each trained policy.


\textbf{Results.} In \Cref{fig:rl}, we show training rewards (left), the standard deviation of rewards within a group during training (middle), and evaluation results (right). In the training plots, we indicate the number of completed steps (within a fixed compute budget) using vertical dashed lines.
We observe that TCT completes the most GRPO training steps. For example, on GSM8K, TCT completes around $44$K steps compared to only around $21$K steps for AR. This is a direct consequence of the adaptive nature of the TCT and Fast-dLLM samplers, which generate multiple tokens in-parallel, hence leading to much faster generation of GRPO rollouts compared to one-token-at-a-time, left-to-right autoregressive sampling. Next, we observe that TCT exhibits lower training rewards compared to both AR and confidence-thresholding (Fast-dLLM). Note, however, that training rewards are not directly comparable, as different sampling strategies are used for rollout generation. Hence, the lower training reward of TCT can be directly attributed to its slightly lower pass@1 rates (\Cref{fig:llada_tct_all_ds}). More important than the average reward for RL is the diversity of rewards within a group of rollouts used to compute GRPO advantages. In the extreme case where all samples receive the same reward, all advantages are equal to zero, and there is no learning signal for RL to exploit. Encouragingly, we find that TCT helps preserve high rollout diversity, as it exhibits the highest standard deviation of rewards throughout training. 

Lastly, at test time, we observe that the policies obtained under TCT or Fast-dLLM sampling yield better evaluation performance compared to AR-trained policies. This complements recent findings in the literature \citep{ni2026flexibility} by demonstrating that when training cost is taken into account, there is no clear advantage to using autoregressive rollouts for GRPO, at least when using the efficient likelihood estimator from \citet{zhao2025d1}. We leave for future work the investigation of whether our findings also hold when using unbiased, yet more computationally expensive, likelihood estimators \citep{turok2026duelexactlikelihoodmasked, ni2026flexibility}.

\begin{figure}
    \centering
    \vspace{-15pt}
    \includegraphics[width=1\linewidth]{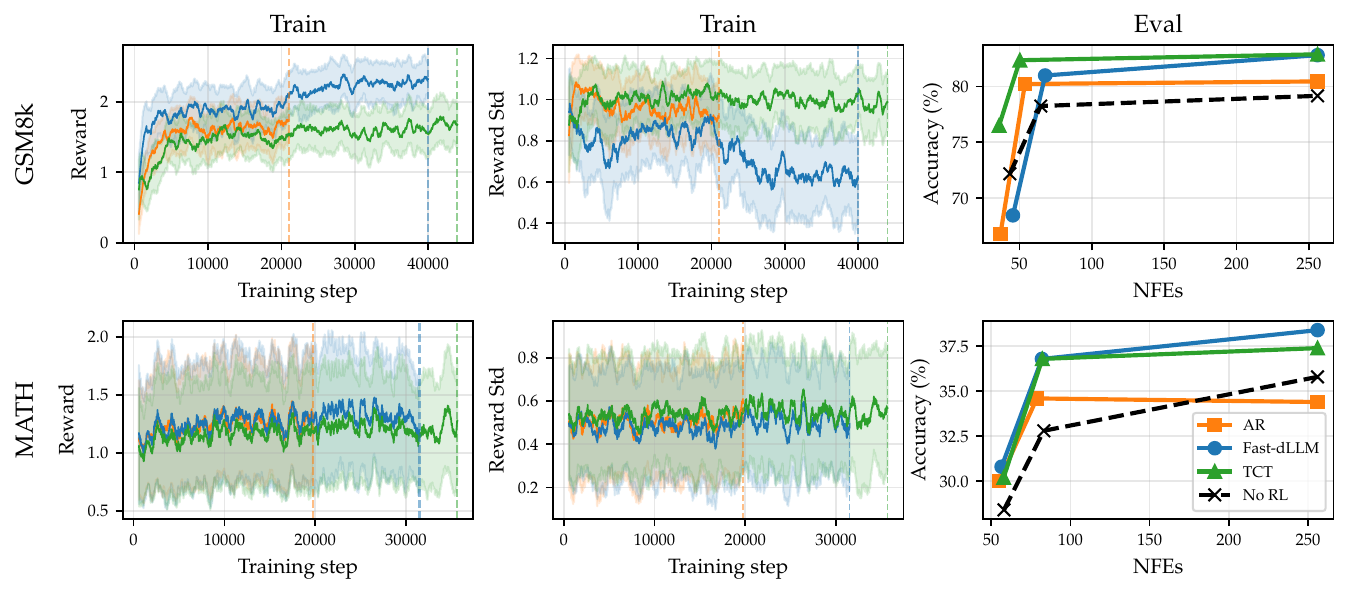}
    \vspace{-10pt}
    \caption{GRPO post-training results for LLaDA-8B-Instruct on GSM8K (top) and MATH (bottom). \emph{Left}: Group mean training reward. \emph{Middle}: Group standard deviation of the reward. \emph{Right}: Evaluation accuracy under greedy decoding with confidence thresholding at varying $\lambda$. All models are trained for a fixed wall-clock budget of 72 hours; vertical dashed lines indicate the number of completed training steps for each sampler. TCT completes around twice as many training steps compared to AR due to its adaptive parallelism, maintains the highest reward diversity throughout training due to $T_{\text{pos}} > 0$, and yields evaluation performance comparable to or better than both AR and Fast-dLLM baselines.}
    \vspace{-10pt}
    \label{fig:rl}
\end{figure}

\section{Related work}
\label{sec:related-work}

While sample diversity is well-studied in the autoregressive setting, it has received comparatively little attention in the context of dLLMs. \citet{ni2026flexibility} and \citet{shen2026improvingdiffusionlanguagemodel} both study how the generation order impacts the reachable solution space, finding that confidence-based heuristics collapse diversity; however, neither proposes a sampling strategy that recovers diversity while retaining the efficiency of parallel generation. Concurrent to our work, \citet{wu2026time} propose TAPS, which injects time-dependent noise into the conditioning signal to encourage semantic branching in early denoising steps, and \citet{lamont2026free} repel samples in feature space during generation to penalize within-batch redundancy. Both are training-free and effective, but operate on different axes than our approach where we modify the remasking order itself. The three approaches are thus complementary. We discuss further related work in Appendix~\ref{app:related_work}.



\section{Conclusion}

We have introduced tempered remasking heuristics for dLLMs,
showing that simple, stochastic relaxations controlled by a single temperature $\tpos$ can recover the rollout diversity needed for effective post-training without sacrificing the advantages of parallel generation.

\textbf{Limitations and future work.}
Our work does not provide a mechanism for how to tune $T_{\text{pos}}$ to optimally balance sample quality and diversity, and except for \Cref{fig:t_pos_ablate} we do not empirically analyze our method's sensitivity to the parameter.
Moreover, we observe that although TCT strongly outperforms Fast-dLLM sampling \citep{wu2025fastdllmtrainingfreeaccelerationdiffusion} in our pass@NFE experiments (\Cref{fig:llada_tct_all_ds}), we find that the two often perform comparably on downstream tasks such as test-time compute (\Cref{fig:ttc}) and RL post-training (\Cref{fig:rl}).
Given that prior work has shown a connection between pass@$k$ and success in these settings for autoregressive models \citep{yue2025limit-of-rlvr},
it would be valuable for future work to investigate why this is not reflected in d1-style \citep{zhao2025d1} GRPO for dLLMs. Additionally, preliminary experiments on a larger set of models (\Cref{fig:tlc_mbpp_sdar_llada2}) suggest that the effectiveness of TLC may depend on the amount and nature of diffusion training received by the underlying model, potentially limiting its benefits for dLLMs obtained through comparatively short diffusion continued pretraining from autoregressive initializations.
Since such conversions are often favored by practitioners seeking to reduce training cost, future work should explore how to make tempered samplers more effective for such models.
Finally, while we take measures to empirically validate the modeling used in our formal analysis (Appendix~\ref{app:assump1_empirical}),
our proofs rely on the idealized assumptions that fork-token entropy degrades linearly as anchors are revealed (Assumption~\ref{assumption:anchor-fork}) and that $q_{\theta}$ preserves the data distribution's fork structure (Assumption~\ref{assumption:model-quality}), neither of which would hold perfectly in practice.
We also only theoretically study TLC at $K = 1$, leaving a theoretical gap to parallel generation unlike other analyses such as those by \citet{wu2025fastdllmtrainingfreeaccelerationdiffusion,ben-hamuAcceleratedSamplingMasked2025}.


\section*{Acknowledgments}
 A. S-L. and TX. O. were supported by NSF grant number CCF-1918839.
 This project was generously supported by the Bosch Center for Artificial Intelligence.

\section*{Use of AI assistance}
We acknowledge the use of AI tools in the preparation of this manuscript. Specifically, we utilized
language models (LMs) to assist with grammar and style editing, preparation of
figures, general coding assistance, drafting and checking proofs, and other tasks. We have reviewed and edited the content to
ensure accuracy and clarity, and we take full responsibility for the final version of the manuscript.

\bibliography{main}
\bibliographystyle{colm2026_conference}

\newpage
\appendix
\section{Extended related work}
\label{app:related_work}
\textbf{Post-training dLLMs.}
Most work on adapting policy optimization to dLLMs has focused on tackling the intractability of the dLLM's likelihood 
(\citealt{zhao2025d1,rojas2025improving,wang2025revolutionizing,wang2025spg,wang2026d2improvedtechniquestraining,turok2026duelexactlikelihoodmasked}; \emph{inter alia}), while the impact of the remasking strategy on policy exploration has received less attention thus far.
Most closely related to our work, \citet{ni2026flexibility} observed a tension between efficiency and diversity
in deterministic confidence-based samplers, and
instead proposed using AR sampling to encourage exploration in dLLMs.
Crystallizing their observation into a formal model directly led to the alternative solution of tempered samplers we propose in this work, which improve sample diversity without sacrificing generation speed. 

\textbf{Remasking strategies for dLLMs.}
Building on the simplicity of the first-hit sampler proposed by 
\citet{zhengMaskedDiffusionModels2025}, a variety of heuristic 
remasking strategies have been proposed.
Fixed-budget methods choose a pre-determined number of positions after ranking candidates by confidence 
\citep{chang2022maskgit,nie2025large}, prediction margin 
\citep{kim2025train}, or KL stability across steps 
\citep{kim2025klass},
while adaptive methods unmask variable-size sets via confidence 
thresholds \citep{wu2025fastdllmtrainingfreeaccelerationdiffusion} 
or entropy bounds
\citep{ben-hamuAcceleratedSamplingMasked2025}.
Orthogonally, some works have sought to \emph{learn} remasking 
strategies 
\citep{huang2025reinforcing,jazbec2025learningunmaskingpoliciesdiffusion,
chen2025dultraultrafastdiffusionlanguage,zhao2026diffpotrainingdiffusionllms}.
Common to all of the above is their focus on the accuracy--speed 
tradeoff for individual samples under greedy decoding, in contrast to our focus on diversity.

\textbf{The role of sample diversity in language modeling.}
Sample diversity has been found to be a critical component in multiple stages of the LLM 
scaling pipeline: \citet{yue2025limit-of-rlvr} showed that RL accuracy 
gains are bounded by base-model pass@$k$, and
test-time scaling requires meaningfully distinct candidates for 
strategies such as self-consistency 
\citep{wang2023selfconsistencyimproveschainthought} and 
Best-of-N \citep{snell2025scaling}.
In autoregressive models, maintaining diversity during post-training has required 
indirect interventions such as entropy regularization 
\citep{yu2025dapoopensourcellmreinforcement,
cui2025entropymechanismreinforcementlearning,
lipkin2026entropypreserving}, entropy-based advantage shaping \citep{cheng2026reasoning}, or selective regularization of low-probability exploratory tokens \citep{huang2025low};
we show that in the case of dLLMs, the flexible generation order itself provides an alternative lever.

\textbf{Fork tokens.} \citet{bigelow2024forking} studied forking tokens in neural text generation, showing that resampling a small set of high-impact tokens can induce qualitatively different continuations. More recently, \citet{wang2025beyond} popularized the term \emph{forking tokens} to describe high-entropy tokens that steer reasoning, and demonstrated their crucial role for effective RLVR of autoregressive LLMs. \citet{ni2026flexibility} investigated forks in dLLMs and empirically showed that their entropy degrades under popular confidence-based samplers. In our work, we formalize these fork tokens operationally to better understand how the choice of remasking strategy impacts sample diversity for dLLMs.

\newpage
\section{Algorithms}
\label{appendix:algos}


\begin{figure}[h]
\begin{minipage}[t]{0.48\textwidth}
\begin{algorithm}[H]
\caption{Low-Confidence (LC) Sampling}
\label{alg:lc}
\begin{algorithmic}[1]
\STATE \textbf{Input:} $\bm{x}_t$, $q_\theta$, $K$, $T_\text{token}$
\STATE \textbf{Output:} $\bm{x}_{t-1}$
\STATE $\mathcal{M}_t \gets \{k \in [L] \mid x_t^k = \texttt{[MASK]}\}$
\STATE $x_0^k \sim q_\theta^k(\cdot \mid \bm{x}_t;\, T_\text{token})$
\STATE $c_t^k \gets q_\theta^k(x_0^k \mid \bm{x}_t)$
\STATE $\mathcal{U}_t \gets \arg\!\operatorname{top\text{-}K}_{k \in \mathcal{M}_t}\, c_t^k$
\STATE $x_{t-1}^k := \begin{cases} x_0^k \:, & k \in \mathcal{U}_t \\ x_t^k \:, & \text{else} \end{cases}$
\end{algorithmic}
\end{algorithm}
\end{minipage}
\hfill
\begin{minipage}[t]{0.48\textwidth}
\begin{algorithm}[H]
\caption{Tempered Low-Confidence (TLC) Sampling}
\label{alg:tlc}
\begin{algorithmic}[1]
\STATE \textbf{Input:} $\bm{x}_t$, $q_\theta$, $K$, $T_\text{token}$, \colorbox{cyan!15}{$T_\text{pos}$}
\STATE \textbf{Output:} $\bm{x}_{t-1}$
\STATE $\mathcal{M}_t \gets \{k \in [L] \mid x_t^k = \texttt{[MASK]}\}$
\STATE $x_0^k \sim q_\theta^k(\cdot \mid \bm{x}_t;\, T_\text{token})$
\STATE $c_t^k \gets  q_\theta^k(x_0^k \mid \bm{x}_t)$
\STATE \colorbox{cyan!15}{$\tilde{c}_t^k \propto (c_t^k)^{1/T_\text{pos}} \cdot \indicator{k \in \mathcal M_t}$}
\STATE \colorbox{cyan!15}{$\mathcal{U}_t \sim \text{Cat}(\bm{\tilde{c}}_{t})$ w/o repl., $|\mathcal{U}_t| = K$}
\STATE $x_{t-1}^k := \begin{cases} x_0^k \:, & k \in \mathcal{U}_t \\ x_t^k \:, & \text{else} \end{cases}$
\end{algorithmic}
\end{algorithm}
\end{minipage}
\end{figure}

\begin{figure}[h]
\begin{minipage}[t]{0.48\textwidth}
\begin{algorithm}[H]
\caption{Confidence Thresholding (CT) Sampling; Fast-dLLM \citep{wu2025fastdllmtrainingfreeaccelerationdiffusion}}
\label{alg:ct}
\begin{algorithmic}[1]
\STATE \textbf{Input:} $\bm{x}_t$, $q_\theta$, $\lambda$, $T_\text{token}$
\STATE \textbf{Output:} $\bm{x}_{t-1}$
\STATE $\mathcal{M}_t \gets \{k \in [L] \mid x_t^k = \texttt{[MASK]}\}$
\STATE $x_0^k \sim q_\theta^k(\cdot \mid \bm{x}_t;\, T_\text{token})$
\STATE $c_t^k \gets  q_\theta^k(x_0^k \mid \bm{x}_t)$
\STATE $\mathcal{U}_t \gets \{k \in \mathcal{M}_t : c_t^k > \lambda\}$
\STATE $x_{t-1}^k := \begin{cases} x_0^k \:, & k \in \mathcal{U}_t \\ x_t^k \:, & \text{else} \end{cases}$
\end{algorithmic}
\end{algorithm}
\end{minipage}
\hfill
\begin{minipage}[t]{0.48\textwidth}
\begin{algorithm}[H]
\caption{Tempered Confidence Thresholding (TCT) Sampling}
\label{alg:tct}
\begin{algorithmic}[1]
\STATE \textbf{Input:} $\bm{x}_t$, $q_\theta$, $\lambda$, $T_\text{token}$, \colorbox{cyan!15}{$T_\text{pos}$}
\STATE \textbf{Output:} $\bm{x}_{t-1}$
\STATE $\mathcal{M}_t \gets \{k \in [L]\mid x_t^k = \texttt{[MASK]}\}$
\STATE $x_0^k \sim q_\theta^k(\cdot \mid \bm{x}_t;\, T_\text{token})$
\STATE $c_t^k \gets q_\theta^k(x_0^k \mid \bm{x}_t)$
\STATE \colorbox{cyan!15}{$b_t^k \sim \text{Ber}\!\left(\sigma\!\left((c_t^k - \lambda)/T_\text{pos}\right)\right)$}
\STATE \colorbox{cyan!15}{$\mathcal{U}_t \gets \{k \in \mathcal{M}_t : b_t^k = 1\}$}
\STATE $x_{t-1}^k := \begin{cases} x_0^k \:, & k \in \mathcal{U}_t \\ x_t^k \:, & \text{else} \end{cases}$
\end{algorithmic}
\end{algorithm}
\end{minipage}
\end{figure}

\newpage

\section{Formal statements and proofs}
\label{appendix:proofs}

In this appendix we develop the formal machinery described in \Cref{sec:methods_forks}.

\textbf{Roadmap.}
We begin by giving an operational definition of fork tokens (\Cref{def:fork-token}) and state the
main assumption underlying our analysis
(\Cref{assumption:anchor-fork}).
We then establish a pairwise ordering lemma between anchors and forks (\Cref{lem:pairwise})
and use it to prove our main result: increasing $\tpos$ increases
fork-token entropy (\Cref{prop:fork-entropy}).
Next, \Cref{cor:semantic-entropy}
ties this to the semantic sample entropy
through an additional assumption on the model quality (\Cref{assumption:model-quality}).
Finally, we show that under a mild additional assumption token temperature does not help preserve fork token entropy
(\Cref{prop:token-temp}) while autoregressive generation provably does, compared to low-confidence remasking (\Cref{prop:ar-diversity}).

\textbf{Notation.}
\begin{itemize}
    \item $\bm{x}_t$: a (possibly partially masked) string of length $L$ at time $t$.
    \item $\mathcal{M}_t = \{k \in [L] \mid x_t^k = \mask\}$: the set
          of masked positions in $\bm{x}_t$.
    \item $q_\theta$: the dLLM used.
    \item $q_\theta^k(\cdot | \bm{x}_t)$: its predicted marginal distribution over tokens at
          position $k$, when prompted with $\bm x_t$.
    \item $c_t^k$ or $c_{t, k}$: the confidence
          of the model $q_{\theta}$ at position $k$ in state $\bm{x}_t$ (see \Cref{sec:bg-remasking}). When using the superscript for other purposes (such as taking powers), we use the double subscript notation.
    \item $\tpos$: position temperature, controlling the stochasticity
          of the remasking order for TLC.
    \item $P^{\pi, \theta}$ or $P_{\tpos}$: the joint distribution
          over generation trajectories induced by pairing $q_\theta$ with
          remasking strategy $\pi$. The shorthand $P_{\tpos}$ is
          used when $\pi$ is TLC with position temperature $\tpos$.
    \item $T(j)$ or $T_\pi(j)$: the step at which position $j$ is
          unmasked. Since this is a random quantity in $P^{\pi, \theta}$, the subscript $\pi$ is included when the remasking
          strategy used needs to be made explicit.
    \item $H_{P^{\pi,\theta}}(x_0^\ell\mid\bm x_t)$: the entropy of the token generated at position $\ell$, under the generative process induced by $\pi$ and $q_\theta$ starting from $\bm{x}_t$.
    \item $\bar H_{P^{\pi, \theta}}(x_0^\ell \mid \bm x_t)$:
   the entropy of the model $q_\theta$'s predictive distribution at $\ell$ at the step where $\ell$ is unmasked, averaged over trajectories starting from $\bm{x}_t$:
          \begin{equation*}
          \bar H_{P^{\pi, \theta}}(x_0^\ell \mid \bm x_t) \triangleq \mathbb{E}_{\bm x_{T_\pi(\ell)-1} \sim P^{\pi, \theta}(\cdot \mid \bm x_t)}\left[ H_{q^\ell_\theta}(x_0^\ell \mid \bm x_{T_\pi(\ell)-1}) \right] \; .
          \end{equation*}
    \item $\llbracket \cdot \rrbracket : \mathcal V^L \to \mathcal S$:
          a function mapping strings to semantic outcomes (e.g., final
          answers for mathematics problems, or program outputs with
          respect to a fixed test suite).
\end{itemize}

With this notation in hand, we can now give an operational definition of fork tokens.

\begin{definition}[$(\epsilon, \delta)$-fork token]
\label{def:fork-token}
Let $\bm x_t$ be a partially masked string with masked positions
$\mathcal M_t$.
A position $\ell \in \mathcal M_t$ is an
$(\epsilon, \delta)$\emph{-fork token}
with respect to $p_\text{data}$ and $\bm x_t$ if:
\begin{enumerate}
    \item
    $H_{p_\text{data}}(x_0^\ell \mid \bm x_t,
    \llbracket \bm x_0 \rrbracket) \leq \epsilon$\,;
    \item
    $H_{p_\text{data}}(\llbracket \bm x_0 \rrbracket
    \mid \bm x_t, x_0^\ell) \leq \delta$\,.
\end{enumerate}
\end{definition}
That is, the fork token's value and the semantic outcome are
tightly coupled: knowing one constrains the other to a high degree of certainty.
Repeatedly applying the chain rule, it follows immediately that
\begin{equation}
\label{eq:sandwich}
H_{p_\text{data}}(x_0^\ell \mid \bm x_t) - \epsilon
\leq
H_{p_\text{data}}(\llbracket \bm x_0 \rrbracket \mid \bm x_t)
\leq
H_{p_\text{data}}(x_0^\ell \mid \bm x_t) + \delta \; .
\end{equation}
This aligns with the intuition in prior work \citep{wang2025beyond,ni2026flexibility} that fork-token entropy largely dictates the semantic diversity of samples.

To relate this definition to \citet{ni2026flexibility}'s empirical observations,
we need a model of \emph{how} the distribution $P^{\pi, \theta}$ induced by pairing $q_\theta$ with a remasking function $\pi$ interacts with such fork tokens.
Based on their findings,
we posit a simple model in which a fork token's predictive
entropy is governed by a small set of high-confidence
\emph{anchor} positions whose confidences dominate that of the
fork (see Appendix~\ref{app:assump1_empirical} for empirical validation).

\begin{assumption}[Anchor-fork degeneracy]
\label{assumption:anchor-fork}
Let $\ell$ be an $(\epsilon, \delta)$-fork token with respect to
$p_\text{data}$, $\llbracket \cdot \rrbracket$, and $\bm x_t$.
We say that $P^{\pi, \theta}$ suffers from an \emph{anchor-fork
degeneracy} at $\ell$ in $\bm x_t$ if there exist anchor
positions $\mathcal A \subset \mathcal M_t \setminus \{\ell\}$
with constants $\eta_a > 0$ such
that
for every state $\bm x_{t'}$
reachable through $P^{\pi, \theta}(\cdot \mid \bm x_t)$ with
$\ell \in \mathcal M_{t'}$:
\begin{enumerate}
    \item \textbf{There is a persistent anchor-fork confidence gap:}
    $c_{t'}^a > c_{t'}^\ell$
    for every remaining anchor $a \in \mathcal A \cap \mathcal M_{t'}$.
    \item \textbf{Revealing anchors linearly degrades the fork entropy:}
    \begin{equation}
    \label{eq:linear-entropy-reduction}
    H_{q^\ell_\theta}(x_0^\ell \mid \bm x_{t'})
    = H_{q^\ell_\theta}(x_0^\ell \mid \bm x_t)
    - \sum_{a \in \mathcal A \setminus \mathcal M_{t'}} \eta_a
    \end{equation}
    where
    $H_{q^\ell_\theta}(x^\ell_0 \mid \bm x_t) > \sum_{a \in \mathcal A} \eta_a > 0$.
\end{enumerate}
\end{assumption}

As mentioned in \Cref{sec:methods_forks},
this model is directly inspired by (and consistent with)
the empirical findings of
\citet{ni2026flexibility}:
(i)~LC deterministically reveals all anchors before the fork, yielding
minimal entropy;
(ii)~$T_\text{token}$ will not affect the unmasking \emph{order} unless it flips the confidences, and hence often leaves the expectation over $\pi$ unchanged
(Appendix~\ref{app:token-temp});
and (iii)~left-to-right generation yields strictly higher fork
entropy than LC whenever at least one anchor follows the fork in
positional order (Appendix~\ref{app:ar-diversity}).

We now give formal evidence for each of these claims,
beginning with a pairwise ordering bound that underlies the main proposition.

\subsection{A pairwise anchor-fork ordering bound}
\label{app:lemma}

\begin{lemma}[Pairwise anchor-fork ordering bound]
\label{lem:pairwise}
Let $\ell$ be a fork token subject to an anchor-fork degeneracy
(Assumption~\ref{assumption:anchor-fork}),
and let $a \in \mathcal A$ be one of its anchors.
Under TLC with $K=1$ at position temperature $\tpos > 0$,
let $T(\ell)$ and $T(a)$ denote the steps at which $\ell$ and $a$
are unmasked, respectively.
Define
$$
\delta_a \triangleq
  \inf_{\bm x_{t'}}
    \ln \frac{c_{t', a}}{c_{t', \ell}}
\qquad \text{and} \qquad
\Delta_a \triangleq
  \sup_{\bm x_{t'}}
    \ln \frac{c_{t', a}}{c_{t', \ell}}
$$
where both extrema are taken over all states $\bm x_{t'}$
reachable from $\bm x_t$ with $\{\ell, a\} \subset \mathcal M_{t'}$.
Then
$$
\sigma\!\left(\frac{\delta_a}{\tpos}\right)
\leq
P_{\tpos}\!\left(T(a) < T(\ell)\right)
\leq
\sigma\!\left(\frac{\Delta_a}{\tpos}\right)
$$
where
$\sigma(x) = 1/(1+e^{-x})$ is the sigmoid function.
\end{lemma}

\begin{proof}
Note first that $0 < \delta_a \leq \Delta_a < \infty$:
the lower bound follows from the confidence gap
(Assumption~\ref{assumption:anchor-fork}, condition~1),
and the upper bound from the fact that $c_{t', a} \leq 1$ and
$c_{t', \ell} > 0$ in every reachable state.
The statement is thus well-posed.

Now, under TLC with $K = 1$, each step selects a single masked position
$j$ with probability proportional to $c_{t, j}^{1/\tpos}$.
Let $t^* = \min(T(a), T(\ell))$ be the first step at which
either $a$ or $\ell$ is unmasked.
At step $t^*$, conditioning on the event that the selected position
lies in $\{a, \ell\}$ and marginalizing out all other
masked positions, we obtain
\begin{align*}
P_{\tpos}(T(a) < T(\ell))
&= \frac{c_{t^*, a}^{1/\tpos}}
        {c_{t^*, a}^{1/\tpos} + c_{t^*, \ell}^{1/\tpos}}
= \frac{(c_{t^*, a}/c_{t^*, \ell})^{1/\tpos}}
       {1 + (c_{t^*, a}/c_{t^*, \ell})^{1/\tpos}}
= \sigma\!\left(
    \frac{\ln(c_{t^*, a}/c_{t^*, \ell})}{\tpos}
  \right).
\end{align*}
Since
$\delta_a
 \leq \ln(c_{t^*, a}/c_{t^*, \ell})
 \leq \Delta_a$
and $\sigma$ is strictly increasing, the result follows.
\end{proof}

\subsection{Increasing $\tpos$ increases fork-token entropy}
\label{app:main}

\begin{proposition}[Increasing $\tpos$ increases fork-token entropy]
\label{prop:fork-entropy}
Let $\tpos, \tpos' > 0$ be two position temperatures for TLC at $K=1$.
Let $\ell$ be a fork token in state $\bm x_t$
subject to an anchor-fork degeneracy
(Assumption~\ref{assumption:anchor-fork})
under both $P_{\tpos}$ and $P_{\tpos'}$.
Let $\delta_a, \Delta_a$ be as in Lemma~\ref{lem:pairwise},
and define
$$
\delta \triangleq \min_{a \in \mathcal A} \delta_a,
\qquad
\Delta \triangleq \max_{a \in \mathcal A} \Delta_a.
$$
Then if
$\tpos' > \tpos \cdot \Delta / \delta$,
$$
\bar H_{P_{\tpos}}(x^\ell_0 \mid \bm x_t)
<
\bar H_{P_{\tpos'}}(x^\ell_0 \mid \bm x_t).
$$
\end{proposition}

\begin{proof}
By the linear entropy reduction condition
(Assumption~\ref{assumption:anchor-fork}, condition~2) and the linearity of expectation, we have for $\tau \in \{T_{\text{pos}}, T_{\text{pos}}' \}$ that
\begin{align*}
\bar H_{P^{\tau, \theta}}(x^\ell_0 \mid \bm x_t)
&= \mathbb{E}_{\bm x_{T_\tau(\ell)-1}\sim P^{\tau,\theta}}
\left[
H_{q^\ell_\theta}
(x_0^\ell \mid \bm x_{T_\tau(\ell)-1})
\right] \qquad \text{(by definition of $\bar H$)}
\\
&= \mathbb{E}_{\bm x_{T_\tau(\ell)-1}\sim P^{\tau,\theta}}
\left[
H_{q^\ell_\theta}(x^\ell_0 \mid \bm x_t)
- \sum_{a \in \mathcal A \setminus \mathcal M_{T_\tau(\ell)-1}} \eta_a
\right]
\qquad \text{(\Cref{assumption:anchor-fork})}
\\
&=
H_{q^\ell_\theta}(x^\ell_0 \mid \bm x_t)
-
\mathbb{E}_{\bm x_{T_\tau(\ell)-1}\sim P^{\tau,\theta}}
\left[
\sum_{a \in \mathcal A \setminus \mathcal M_{T_\tau(\ell)-1}} \eta_a
\right]
\\
&=
H_{q^\ell_\theta}(x^\ell_0 \mid \bm x_t)
-
\sum_{a \in \mathcal A}
P_{\tau}(T(a) < T(\ell)) \cdot \eta_a
\end{align*}
where $T_{\tau}(\ell)$ denotes the unmasking time of the fork. Since each $\eta > 0$, it thus suffices to show that
$P_{\tpos'}(T(a) < T(\ell)) < P_{\tpos}(T(a) < T(\ell))$
for every $a \in \mathcal A$.
Per Lemma~\ref{lem:pairwise} and the definitions of
$\delta, \Delta$:
$$
P_{\tpos'}(T(a) < T(\ell))
\leq
\sigma\!\left(\frac{\Delta}{\tpos'}\right)
<
\sigma\!\left(\frac{\delta}{\tpos}\right)
\leq
P_{\tpos}(T(a) < T(\ell)),
$$
where the strict inequality follows from the gap condition
$\tpos' > \tpos \cdot \Delta / \delta$,
which gives $\Delta / \tpos' < \delta / \tpos$,
combined with $\sigma$ being strictly increasing.
The result follows.
\end{proof}

\begin{remark}
\label{rem:tpos-zero}
\Cref{lem:pairwise}
only applies when $\tpos, \tpos' > 0$, so this proof does not allow us to compare TLC (at some $\tpos > 0$) vs LC.
However, note that this case is trivial:
Under LC, all anchors are revealed before the fork with probability 1 by the confidence gap (Assumption~\ref{assumption:anchor-fork}, condition~1), so
$$
\bar H_{P^{0,\theta}}(x_0^\ell \mid \bm x_t) = H_{q^\ell_\theta}(x_0^\ell \mid \bm x_t) - \sum_{a \in \mathcal A} \eta_a,
$$
the minimal fork entropy achievable under the anchor-fork degeneracy model.
Since any $\tpos > 0$ yields $P_{\tpos}(T(a) < T(\ell)) < 1$ for every anchor $a$ (that is, $\sigma(x)< 1$ for all finite $x$), it follows that
$\bar H_{P^{\tpos,\theta}}(x_0^\ell \mid \bm x_t) > \bar H_{P^{\text{LC},\theta}}(x_0^\ell \mid \bm x_t)$
for any $\tpos > 0$.
That is, any amount of tempering strictly improves expected fork-token entropy over deterministic LC.
\end{remark}

In order to tie this proposition back to \Cref{def:fork-token}
and conclude that the semantic entropy of the final generation will increase (\Cref{cor:semantic-entropy}),
we need one more assumption: that $q_\theta$ captures the data distribution well enough that no matter what $\pi$ is (out of the strategies under consideration), it will not stop $\ell$ from also being a fork token with respect to $P^{\pi, \theta}$.

\begin{assumption}[Model quality]
\label{assumption:model-quality}
Let $\ell$ be an $(\epsilon, \delta)$-fork token with respect to
$p_\text{data}$ and $\bm x_t$.
We say that $q_\theta$ \emph{preserves the fork structure} at
$\ell$ if $\ell$ is also an $(\epsilon, \delta)$-fork token
with respect to $P^{\pi, \theta}$ for all remasking strategies
$\pi$.
\end{assumption}

That is, if the data distribution has a fork at position $\ell$,
then $q_\theta$'s sampling distribution preserves that property
regardless of the remasking strategy, so that the sandwich
bound~\eqref{eq:sandwich} also holds under $P^{\pi, \theta}$.

We are now ready to finally make a statement about how $H_{P^{\pi, \theta}} (\llbracket \bm x_0 \rrbracket \mid \bm x_t)$ varies with $\tpos$: that is, how changing the position temperature affects the semantic entropy of our samples.

\begin{corollary}[Semantic entropy increases with $\tpos$]
\label{cor:semantic-entropy}
If $q_\theta$ preserves the fork structure
at $\ell$ (Assumption~\ref{assumption:model-quality})
and 
$\tpos,\tpos'$
are such that the gap in Proposition~\ref{prop:fork-entropy} satisfies
$$
\bar H_{P_{\tpos'}}(x^\ell_0 \mid \bm x_t)
- \bar H_{P_{\tpos}}(x^\ell_0 \mid \bm x_t)
> \epsilon + \delta + \gamma(\tpos) + \gamma(\tpos') \; ,
$$
where
$\gamma(\tau) \triangleq 
\big| H_{P_{\tau}}(x^\ell_0 \mid \bm x_t) - \bar H_{P_{\tau}}(x^\ell_0 \mid \bm x_t) \big|,
$
then
$$
H_{P_{\tpos'}}(\llbracket \bm x_0 \rrbracket \mid \bm x_t)
>
H_{P_{\tpos}}(\llbracket \bm x_0 \rrbracket \mid \bm x_t).
$$
\end{corollary}

\begin{proof}
Applying the sandwich bound~\eqref{eq:sandwich} under
$P^{\tpos', \theta}$ and $P^{\tpos, \theta}$ respectively
(which is valid by Assumption~\ref{assumption:model-quality}),
together with the definition of $\gamma$:
\begin{align*}
H_{P_{\tpos'}}(\llbracket \bm x_0 \rrbracket \mid \bm x_t)
&\geq H_{P_{\tpos'}}(x^\ell_0 \mid \bm x_t) - \epsilon
&& \text{(sandwich under $P_{\tpos'}$)} \\
&\geq \bar H_{P_{\tpos'}}(x^\ell_0 \mid \bm x_t) - \gamma(\tpos') - \epsilon
&& \text{(definition of $\gamma(\tpos')$)} \\
&> \bar H_{P_{\tpos}}(x^\ell_0 \mid \bm x_t) + \gamma(\tpos) + \delta
&& \text{(by the hypothesis)} \\
&\geq H_{P_{\tpos}}(x^\ell_0 \mid \bm x_t) + \delta
&& \text{(definition of $\gamma(\tpos)$)} \\
&\geq H_{P_{\tpos}}(\llbracket \bm x_0 \rrbracket \mid \bm x_t)
&& \text{(sandwich under $P_{\tpos}$)}
\end{align*}
as claimed.
\end{proof}

\subsection{Token temperature does not affect fork-token entropy under deterministic remasking}
\label{app:token-temp}

As shown in Appendix~\ref{appendix:algos}, the confidences $c_t^k$ do not depend on $T_\text{token}$ as typically implemented\footnote{See, e.g., \href{https://github.com/ML-GSAI/LLaDA/blob/570f29032d6824ea14977c89a8eb402e6eb25f96/generate.py\#L101}{the source code for LLaDA's low-confidence remasking} or \href{https://github.com/NVlabs/Fast-dLLM/blob/af42f511f705a2db0d16aa63ebcfcbc26f29181d/llada/generate.py\#L320}{Fast-dLLM's reference implementation for confidence-thresholding}.} in confidence-based heuristics.
One may speculate that allowing $T_\text{token}$ to affect the confidences and then increasing it would yield similar diversity gains to adjusting $\tpos$.
We here show that, under a strengthened version of the confidence-gap from \Cref{assumption:anchor-fork}, this is provably not the case:
the unmasking order under deterministic LC is invariant to
$T_\text{token}$ regardless, so adjusting $T_\text{token}$ does not help with preserving entropy at forks.

\begin{proposition}[Token temperature invariance under deterministic LC]
\label{prop:token-temp}
Let $\ell$ be a fork token subject to an anchor-fork degeneracy
(Assumption~\ref{assumption:anchor-fork}),
and suppose sampling is performed using deterministic
low-confidence remasking at $K = 1$, with confidences taken post-tempering with $T_\text{token}$.
Suppose further that for all reachable states $\bm x_{t'}$
with $\ell \in \mathcal M_{t'}$, every remaining anchor
$a \in \mathcal A \cap \mathcal M_{t'}$ is such that:
$$
z^a_{(1)} - z^a_{(i)} \geq z^\ell_{(1)} - z^\ell_{(i)}
\qquad \text{for all } i \geq 2, \text{with at least one inequality strict}
$$
where $z^j_{(1)} \geq z^j_{(2)} \geq \cdots$ denotes the logit
vector at position $j$ sorted in descending order.
Then all anchors are revealed before $\ell$ regardless of
$T_\text{token}$.
\end{proposition}

\begin{proof}
Under this version of deterministic LC, the unmasking order is determined 
by $\arg\max_{j \in \mathcal M_t} c^j_t(T_\text{token})$,
where $c^j(T) = \max_v [\mathrm{softmax}(\bm z^j / T)]_v$.
It suffices to show that $c^a(T) > c^\ell(T)$ for all $T > 0$.

Let $V = |\mathcal V|$ and let $\bm z$ be any logit vector with
sorted entries $z_{(1)} \geq z_{(2)} \geq \cdots \geq z_{(V)}$.
Then
$$
c(T) = \frac{e^{z_{(1)}/T}}{\sum_{i=1}^V e^{z_{(i)}/T}}
     = \frac{1}{1 + \sum_{i=2}^V e^{(z_{(i)} - z_{(1)})/T}}.
$$
By the assumption that $z^a_{(1)} - z^a_{(i)} \geq z^\ell_{(1)} - z^\ell_{(i)}$
for all $i \geq 2$, we have
$z^a_{(i)} - z^a_{(1)} \leq z^\ell_{(i)} - z^\ell_{(1)}$
and hence $e^{(z^a_{(i)} - z^a_{(1)})/T} \leq e^{(z^\ell_{(i)} - z^\ell_{(1)})/T}$
for each $i$ and all $T > 0$.
Summing over $i \geq 2$ and using the assumption that at least one inequality is strict (which we note follows anyway from the confidence-gap condition of \Cref{assumption:anchor-fork}), we have that:
$$
c^a(T)
= \frac{1}{1 + \sum_{i=2}^V e^{(z^a_{(i)} - z^a_{(1)})/T}}
>
\frac{1}{1 + \sum_{i=2}^V e^{(z^\ell_{(i)} - z^\ell_{(1)})/T}}
= c^\ell(T).
$$
The unmasking order is thus preserved for all $T > 0$,
so all anchors are revealed before $\ell$ regardless of
$T_\text{token}$.
\end{proof}

\subsection{Autoregressive generation increases fork-token entropy}
\label{app:ar-diversity}

Finally, we observe that autoregressive (left-to-right) generation
yields higher fork-token entropy than deterministic low-confidence
remasking, since it does not systematically defer the fork.

\begin{proposition}[Autoregressive generation improves fork-token entropy]
\label{prop:ar-diversity}
Let $\ell$ be a fork token subject to an anchor-fork degeneracy
(Assumption~\ref{assumption:anchor-fork}) with respect to both LC and AR,
and suppose at least one anchor appears after $\ell$ in
left-to-right order.
Then
$$
\bar H_{P^{\mathrm{LC}, \theta}}(x^\ell_0 \mid \bm x_t)
<
\bar H_{P^{\mathrm{AR}, \theta}}(x^\ell_0 \mid \bm x_t).
$$
\end{proposition}

\begin{proof}
Under both strategies the remasking order is deterministic,
so $R_\pi = \{a \in \mathcal A : T(a) < T(\ell)\}$ is a fixed set
for each.
Under LC, the confidence gap ensures
$R_{\mathrm{LC}} = \mathcal A$.
Under AR, $R_{\mathrm{AR}} = \{a \in \mathcal A : a < \ell\}
\subsetneq \mathcal A$ by assumption.
By the definition of $\bar H$ (as in the proof of Proposition~\ref{prop:fork-entropy}),
$$
\bar H_{P^{\pi, \theta}}(x_0^\ell \mid \bm x_t)
= \mathbb{E}_{\bm x_{T_\pi(\ell)-1}}
\left[
H_{q^\ell_\theta}(x_0^\ell \mid \bm x_{T_\pi(\ell)-1})
\right]
= H_{q^\ell_\theta}(x_0^\ell \mid \bm x_t)
  - \sum_{a \in R_\pi} \eta_a \, ,
$$
where the second equality applies condition~2 of
Assumption~\ref{assumption:anchor-fork}, noting that
$\mathcal A \setminus \mathcal M_{T_\pi(\ell)-1} = R_\pi$.
Since $R_{\mathrm{AR}} \subsetneq R_{\mathrm{LC}} = \mathcal A$
and each $\eta_a > 0$, we thus conclude that
$$
\bar H_{P^{\mathrm{AR}, \theta}}(x_0^\ell \mid \bm x_t)
>
\bar H_{P^{\mathrm{LC}, \theta}}(x_0^\ell \mid \bm x_t).
$$
\end{proof}

\newpage
\subsection{Empirical validation of the formal model}
\label{app:assump1_empirical}

\paragraph{Linear degradation of fork entropy.} Here we conduct an experiment to empirically validate \Cref{assumption:anchor-fork}. The core idea is this: since low-confidence (LC) sampling deterministically unmasks positions in order of their confidence, each token revealed by LC before a fork is a potential anchor under \Cref{assumption:anchor-fork}, condition~1. By measuring whether the fork token's entropy decreases linearly as these anchors are revealed (\Cref{assumption:anchor-fork}, condition~2), we can thus directly test whether the assumed mechanism holds.

Concretely, we first generate one completion per prompt for 100 random prompts per dataset, using AR generation for simplicity. We then bin each position's predictive entropy at the time it was revealed, calculate the 99\% quantile, and take positions exceeding that quantile to be potential forks. We then run generation with LC remasking instead of AR, starting with all positions before (i.e., to the left of) that token already revealed, and log how the fork position's entropy changes until it is selected for unmasking.

The results are shown in \Cref{fig:assump1_empirical}, with one row of figures per dataset. Across all of the datasets, we observe that anchor positions degrade the fork token's entropy at a roughly linear rate. As expected, the entropy does not degrade perfectly linearly, as stated in \Cref{assumption:anchor-fork}; this assumption simply makes the math tractable. However, these results suggest that our theoretical findings carry over to practical gains (as seen in our experiments) via the claimed mechanism.

\paragraph{Semantic diversity of forks.} Moreover, in order to verify that the positions actually correspond to semantic forks, we also conduct a small experiment in which we use deterministic LC generation with the fork position also revealed, set to each one of its top 5 most likely values as predicted by the model. We then measure how many of the 5 completions per fork are well-formed and semantically distinct from each other. The results are shown in \Cref{fig:assump1_sem_div}. Across datasets, between 44\% (GSM8K) and 79\% (MATH) of these positions produce two or more distinct semantic outcomes, far more than one would expect if these were merely uncertain syntactic choices. The lower fraction for GSM8K is likely due to the relative simplicity of the benchmark, causing high pass@1 (that is, high confidence in what the answer should be). 

\begin{figure}[h]
    \centering

    \includegraphics[width=0.96\textwidth]{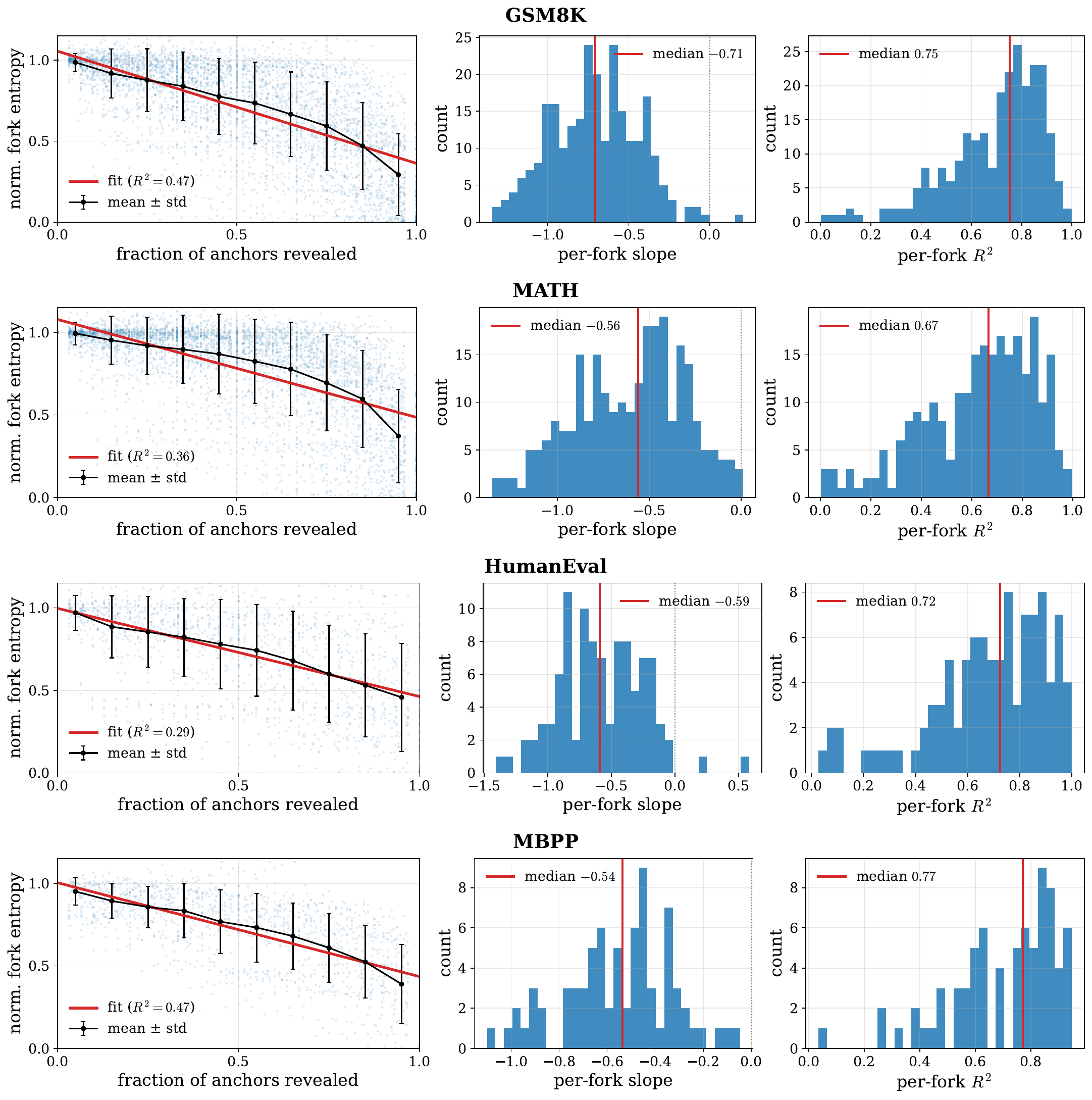}
    \caption{Empirical validation of \Cref{assumption:anchor-fork}. \textit{Left}: Entropy degradation at potential fork tokens as a function of how many potential anchors have been revealed, aggregated over all forks and samples. \textit{Middle}: Slopes of per-fork linear fits. \textit{Right}: $R^2$ correlation between each fork's entropy degradation and a linear fit on that fork's data. In all cases, we see that forks tend to decrease in entropy as their anchors are revealed. While different forks vary in how much their entropies change (causing large error bars and lower $R^2$ on the left), most exhibit an internally nearly-linear relationship (the median per-fork $R^2$ being $\sim0.7$ across datasets in the rightmost plot), and almost all decrease in entropy (middle).}
    \label{fig:assump1_empirical}
\end{figure}

\begin{figure}[h]
    \centering

    \includegraphics[width=0.96\textwidth]{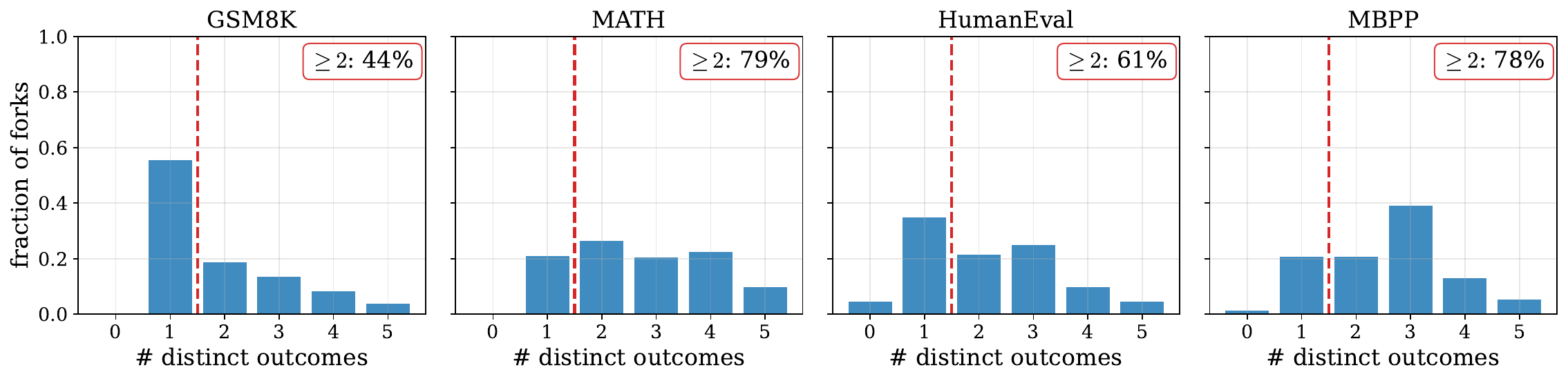}
    \caption{Number of semantically distinct and well-formed outcomes when setting a fork position from \Cref{fig:assump1_empirical} to each of its top 5 most likely values. Overall, a majority of the fork positions produce $\ge 2$ valid and semantically distinct outcomes, indicating that they play an important role in determining the semantics of the outcome (rather than just syntax).}
    \label{fig:assump1_sem_div}
\end{figure}

\clearpage
\section{Additional figures}
\label{appendix:additional-figures}

\begin{figure}[h]
    \centering

    \includegraphics[width=0.5\textwidth]{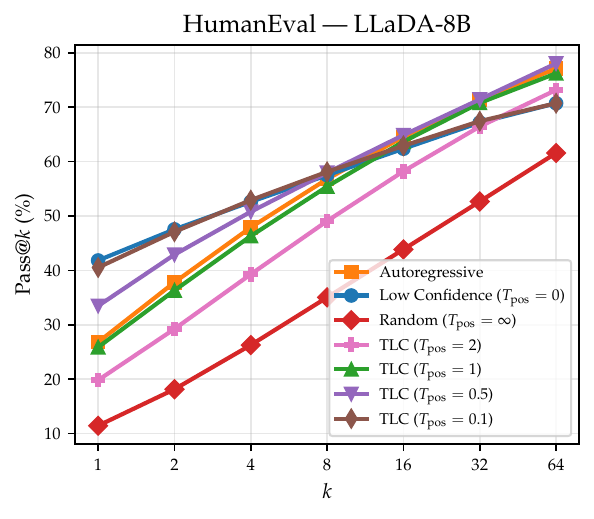}
    \caption{TLC $T_{\text{pos}}$ ablation on HumanEval with LLaDA-8B-Instruct at $T_{\text{token}} = 0.8$. Too little tempering ($T_{\text{pos}} = 0.1$) remains close to low-confidence remasking, while too much ($T_{\text{pos}} = 2$) approaches random remasking and degrades quality. Intermediate values ($T_{\text{pos}} \in \{0.5, 1\}$) strike the best balance, matching or exceeding the autoregressive baseline at high $k$.}
    \label{fig:t_pos_ablate}
\end{figure}

\begin{figure}[h]
    \centering

    \includegraphics[width=\textwidth]{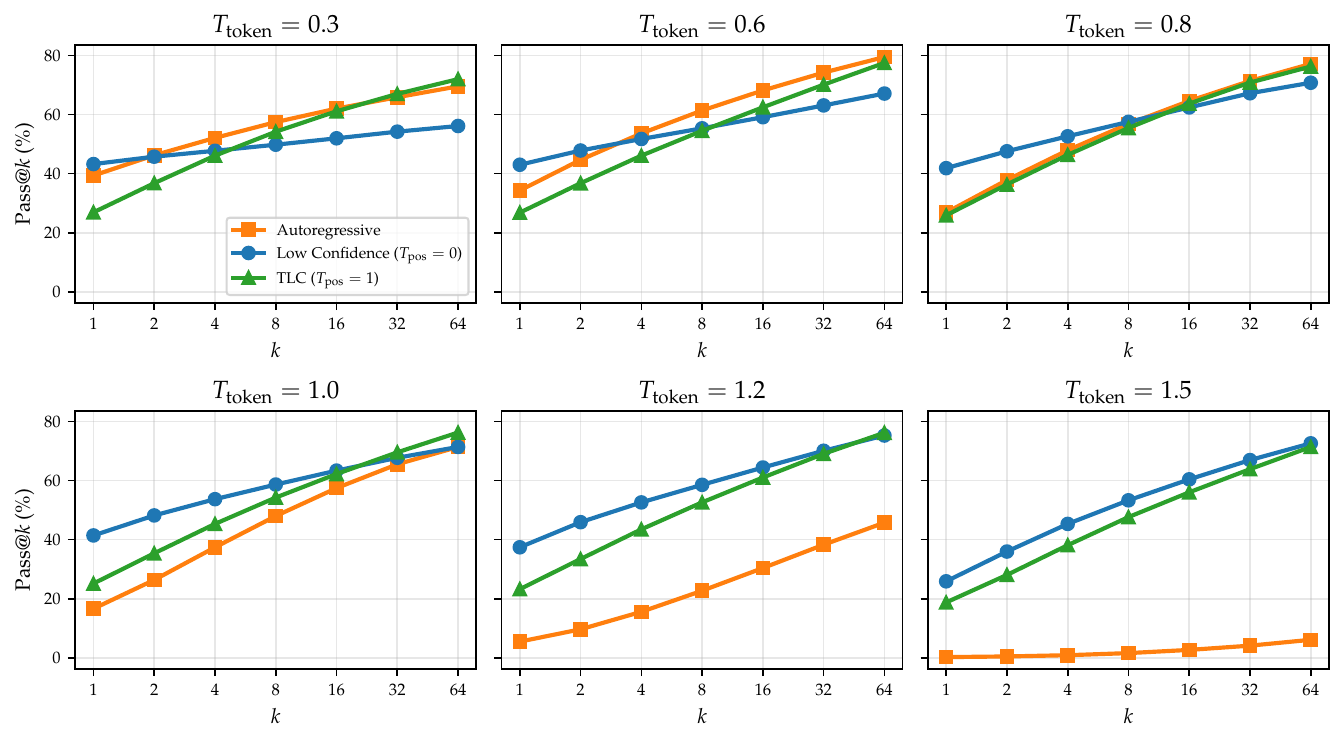}
    \caption{Varying $T_\text{token}$ for LLaDA-8B-Instruct on HumanEval. Autoregressive generation favors low temperatures, obtaining its best pass@64 at $T_\text{token}=0.6$. Meanwhile, low-confidence remasking scales very poorly at low token temperatures due to its greedy position selection. TLC consistently shows strong scaling and most robustness to the exact choice of $T_{\text{token}}$, matching or exceeding the best of autoregressive and low-confidence remasking in pass@64 at all temperatures. Moreover, it exhibits stronger scaling compared to low-confidence across all considered $T_{\text{token}}$, suggesting it achieves diversity through an alternative mechanism.}
    \label{fig:t_token_ablate}
\end{figure}

\begin{figure}[h]
    \centering

    \includegraphics[width=\textwidth]{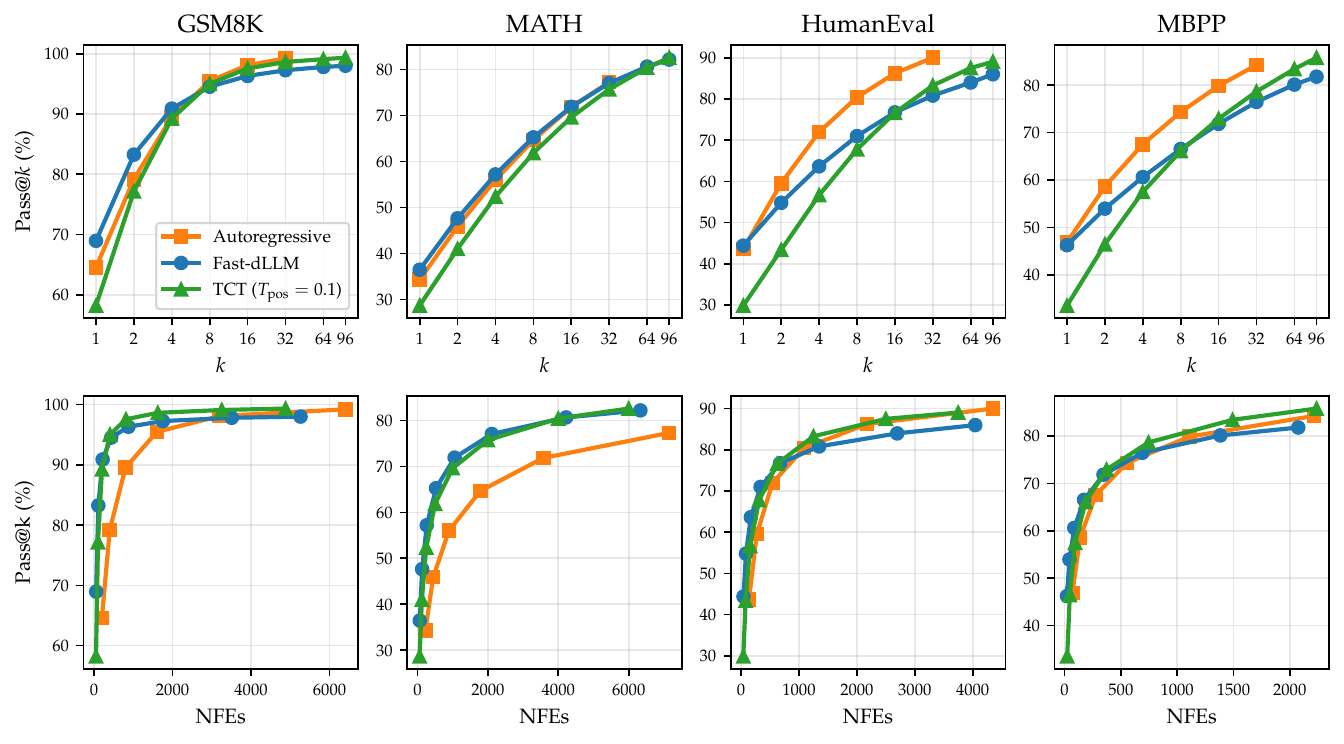}
    \caption{Pass@$k$ (top) and pass@NFE (bottom) for TCT on Dream-7B-Instruct. Note that both rows plot the same data, we just vary the x-axis measure ($k$ vs NFEs). While TCT slightly underperforms AR in pass@$k$, it outperforms AR when the cost of each rollout is taken into account.}
    \label{fig:dream_tct_all_ds}
\end{figure}

\begin{figure}
    \centering
    \includegraphics[width=\linewidth]{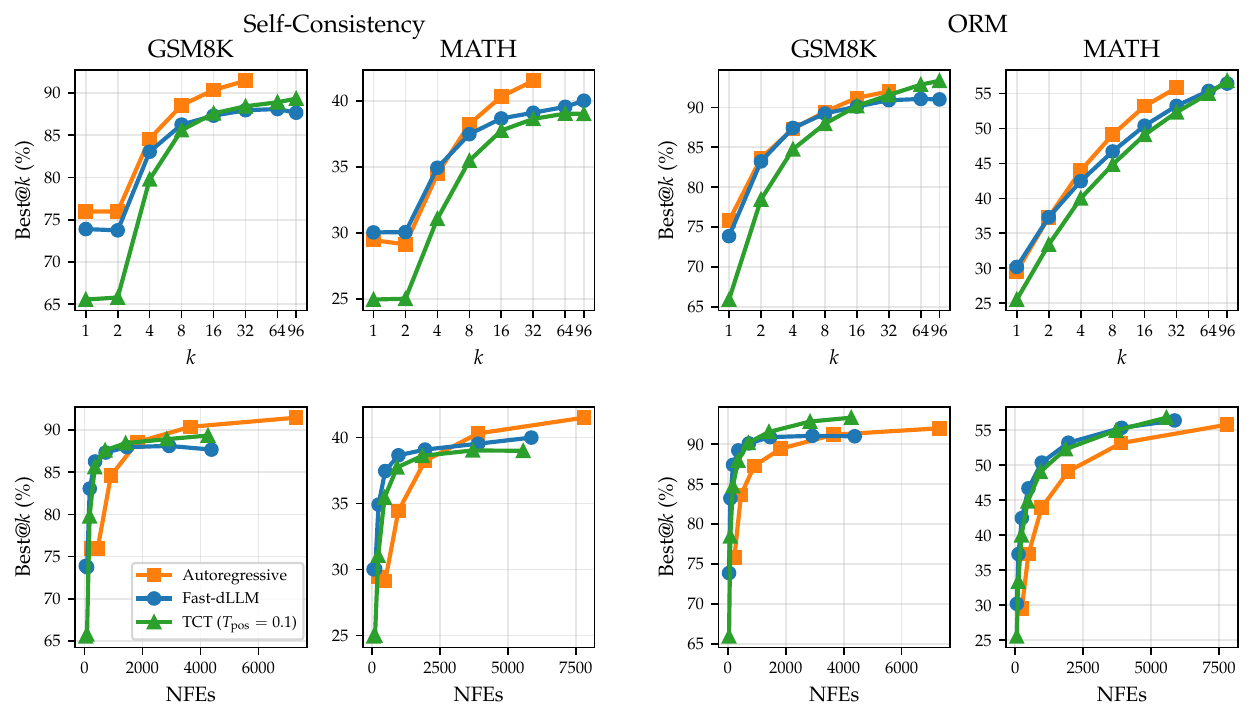}
    \caption{Test-time scaling for LLaDA-8B-Instruct in terms of Best@$k$ (top) and Best@NFE (bottom) when using either self-consistency or an ORM for final answer selection. TCT performs worse compared to AR at higher NFEs when using self-consistency; however, when using an ORM, TCT improves the Pareto frontier compared to autoregressive rollouts.}
    \label{fig:ttc_bestk_bestnfe}
\end{figure}

\begin{figure}
    \centering
    \vspace{-20pt}
    \includegraphics[width=0.75\linewidth]{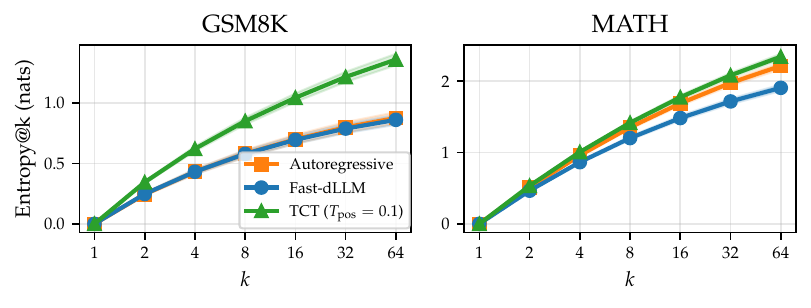}
    \caption{Mean empirical entropy of the per-question answer distributions (see \Cref{fig:ans_dist_gsm8k}) as a function of the group size $k$. TCT consistently yields higher-entropy answer distributions than autoregressive and confidence thresholding, particularly on GSM8K.}
    \label{fig:entropy-vs-k}
\end{figure}

\begin{figure}
    \vspace{-20pt}
    \centering

    \includegraphics[width=\textwidth]{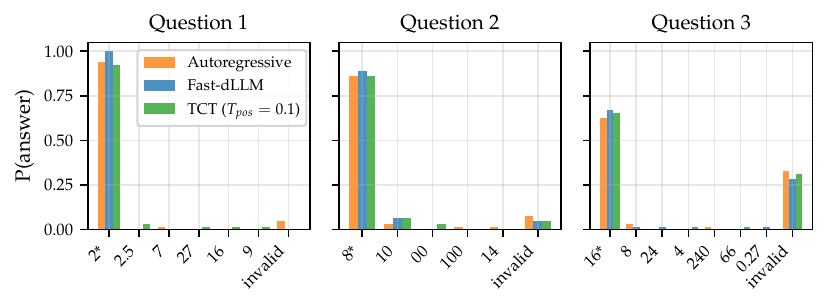}
    \caption{Per-question answer distributions over $k=64$ samples for three GSM8K questions, comparing autoregressive, Fast-dLLM, and TCT sampling. Correct answers are marked with an asterisk (*). We observe that TCT sometimes spreads mass across a wider set of candidates (left). This illustrates the higher-entropy behavior observed in \Cref{fig:entropy-vs-k}: TCT produces more diverse answers, which benefits pass@$k$ and ORM-based selection but can hurt majority voting when the correct answer no longer dominates the distribution.}
    \label{fig:ans_dist_gsm8k}
\end{figure}

\begin{figure}
    \vspace{-20pt}
    \centering

    \includegraphics[width=\textwidth]{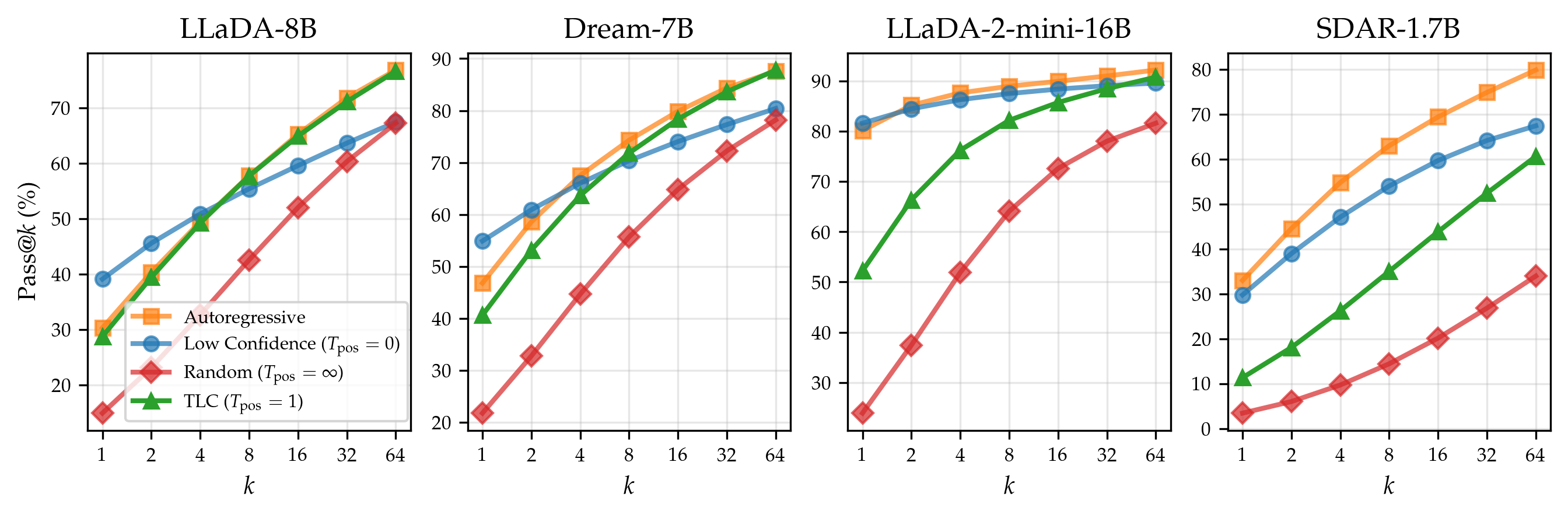}
    \caption{Pass@$k$ for TLC ($T_{\mathrm{pos}}=1$) on MBPP across dLLMs with different amounts of diffusion training. LLaDA-8B was trained as a diffusion model from scratch on approximately 2.3T tokens, whereas Dream-7B, LLaDA-2-mini-16B, and SDAR-1.7B were initialized from autoregressive models and subsequently diffusion-trained on approximately 580B, fewer than 200B, and 50B tokens, respectively. The LLaDA-8B and Dream-7B panels reproduce the MBPP results from \Cref{fig:llada_tcl_all_ds}. TLC closely tracks autoregressive sampling on the more extensively diffusion-trained LLaDA-8B and Dream-7B models, but leaves a larger gap on LLaDA-2-mini-16B and especially SDAR-1.7B.}
    \label{fig:tlc_mbpp_sdar_llada2}
\end{figure}

\begin{figure}[h]
    \centering

    \includegraphics[width=\textwidth]{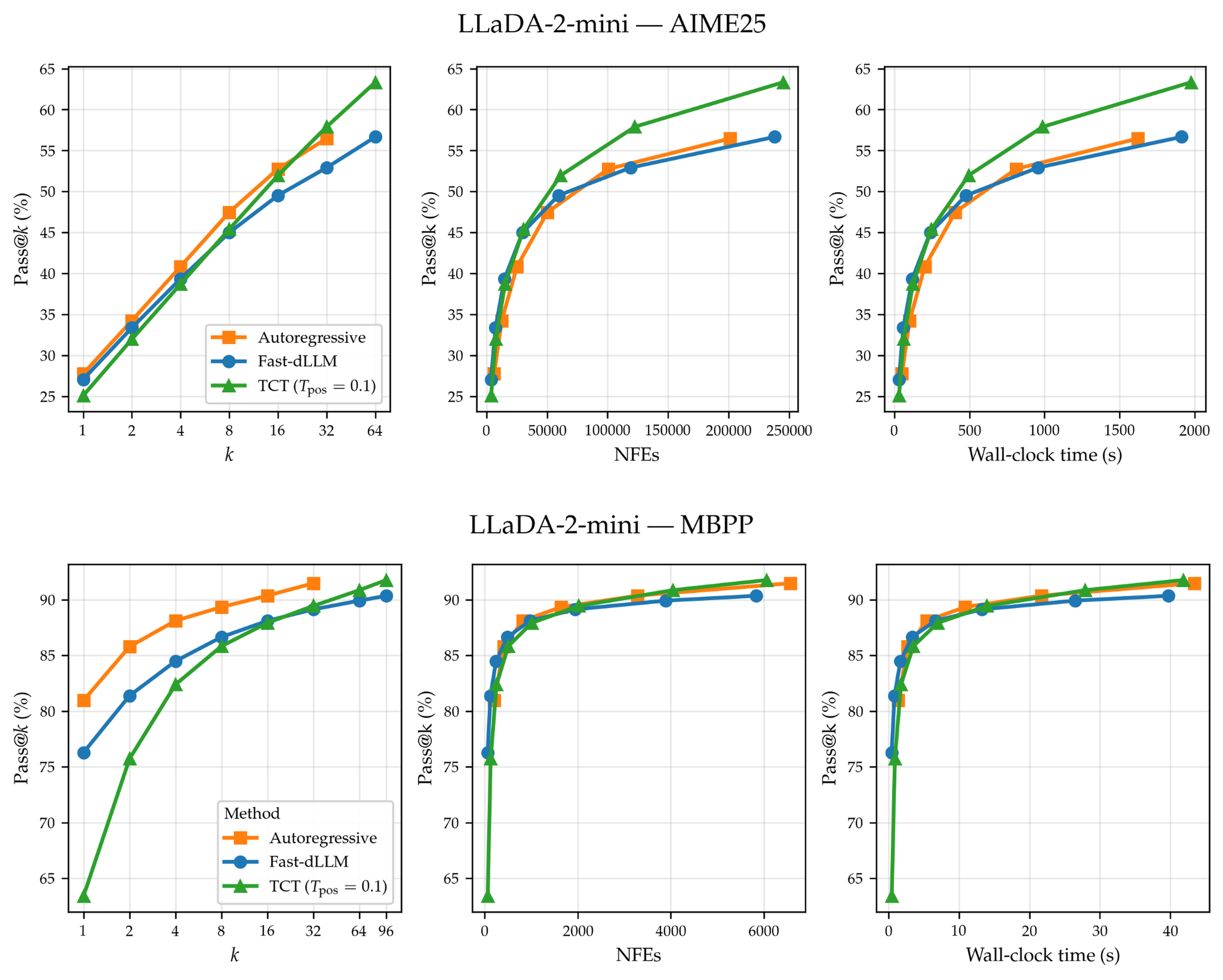}
    \caption{Pass@$k$ and pass@NFE for TCT on LLaDA-2-mini \citep{bie2025llada2}. In addition to pass@$k$ (\textit{Left}) and pass@NFE (\textit{Middle}), we here additionally report pass@wall-clock-time (\textit{Right}) to show that our proposed TCT has a negligible overhead in terms of latency compared to Fast-dLLM.} 
    \label{fig:llada2_tct_all_ds}
\end{figure}

\clearpage
\section{Numerical results}
\label{appendix:numerical-results}
Here we reproduce our main experimental results in a table format (with bootstrapped error bars) for completeness.

\subsection{TLC (\Cref{fig:llada_tcl_all_ds})}
\begin{table}[h]
\centering
\caption{Pass@$k$ (\%) with bootstrapped 95\% confidence intervals (10{,}000 resamples over questions/tasks, 64 samples per question, $T_{\mathrm{token}}=0.8$). Values are mean $\pm$ half-width of the 95\% CI. Best value per column within each dataset block in bold.  Evaluation sets: GSM8K 300 questions, MATH 500, HumanEval 164 tasks, MBPP 500.}
\label{tab:tlc-passk-ci}
\resizebox{\textwidth}{!}{%
\begin{tabular}{llccccccc}
\toprule
& & \multicolumn{7}{c}{$k$} \\
\cmidrule(lr){3-9}
Dataset & Method & 1 & 2 & 4 & 8 & 16 & 32 & 64 \\
\midrule
\multicolumn{9}{c}{\textit{LLaDA-8B-Instruct}} \\
\midrule
\multirow{4}{*}{GSM8K} & Autoregressive & 75.9{\scriptsize$\,\pm\,$3.1} & 86.9{\scriptsize$\,\pm\,$2.7} & \textbf{92.7}{\scriptsize$\,\pm\,$2.1} & \textbf{96.0}{\scriptsize$\,\pm\,$1.7} & \textbf{97.8}{\scriptsize$\,\pm\,$1.4} & \textbf{98.6}{\scriptsize$\,\pm\,$1.2} & 99.0{\scriptsize$\,\pm\,$1.2} \\
 & Low Confidence ($T_{\mathrm{pos}}=0$) & \textbf{78.8}{\scriptsize$\,\pm\,$3.2} & \textbf{87.3}{\scriptsize$\,\pm\,$2.9} & 91.9{\scriptsize$\,\pm\,$2.5} & 94.6{\scriptsize$\,\pm\,$2.1} & 96.3{\scriptsize$\,\pm\,$1.9} & 97.2{\scriptsize$\,\pm\,$1.7} & 98.0{\scriptsize$\,\pm\,$1.5} \\
 & Random ($T_{\mathrm{pos}}=\infty$) & 55.5{\scriptsize$\,\pm\,$2.9} & 73.8{\scriptsize$\,\pm\,$3.1} & 85.9{\scriptsize$\,\pm\,$2.8} & 92.4{\scriptsize$\,\pm\,$2.2} & 95.9{\scriptsize$\,\pm\,$1.7} & 97.9{\scriptsize$\,\pm\,$1.3} & \textbf{99.3}{\scriptsize$\,\pm\,$0.8} \\
 & \textbf{TLC} ($T_{\mathrm{pos}}=1$) & 72.7{\scriptsize$\,\pm\,$3.0} & 85.7{\scriptsize$\,\pm\,$2.7} & 92.3{\scriptsize$\,\pm\,$2.2} & 95.8{\scriptsize$\,\pm\,$1.8} & 97.5{\scriptsize$\,\pm\,$1.5} & 98.4{\scriptsize$\,\pm\,$1.2} & 99.0{\scriptsize$\,\pm\,$1.2} \\
\cmidrule(lr){1-9}
\multirow{4}{*}{MATH} & Autoregressive & 29.6{\scriptsize$\,\pm\,$3.0} & 39.2{\scriptsize$\,\pm\,$3.4} & 48.1{\scriptsize$\,\pm\,$3.7} & 56.5{\scriptsize$\,\pm\,$3.7} & \textbf{64.2}{\scriptsize$\,\pm\,$3.6} & \textbf{70.9}{\scriptsize$\,\pm\,$3.6} & \textbf{77.0}{\scriptsize$\,\pm\,$3.8} \\
 & Low Confidence ($T_{\mathrm{pos}}=0$) & \textbf{34.1}{\scriptsize$\,\pm\,$3.4} & \textbf{42.1}{\scriptsize$\,\pm\,$3.6} & \textbf{49.5}{\scriptsize$\,\pm\,$3.8} & \textbf{56.6}{\scriptsize$\,\pm\,$3.8} & 63.1{\scriptsize$\,\pm\,$3.7} & 68.7{\scriptsize$\,\pm\,$3.7} & 73.4{\scriptsize$\,\pm\,$3.9} \\
 & Random ($T_{\mathrm{pos}}=\infty$) & 20.4{\scriptsize$\,\pm\,$2.3} & 30.1{\scriptsize$\,\pm\,$2.9} & 40.1{\scriptsize$\,\pm\,$3.4} & 49.8{\scriptsize$\,\pm\,$3.6} & 59.1{\scriptsize$\,\pm\,$3.6} & 67.4{\scriptsize$\,\pm\,$3.7} & 74.6{\scriptsize$\,\pm\,$3.8} \\
 & \textbf{TLC} ($T_{\mathrm{pos}}=1$) & 28.4{\scriptsize$\,\pm\,$2.9} & 37.8{\scriptsize$\,\pm\,$3.4} & 46.6{\scriptsize$\,\pm\,$3.6} & 55.0{\scriptsize$\,\pm\,$3.7} & 62.6{\scriptsize$\,\pm\,$3.7} & 69.3{\scriptsize$\,\pm\,$3.7} & 75.2{\scriptsize$\,\pm\,$3.8} \\
\cmidrule(lr){1-9}
\multirow{4}{*}{HumanEval} & Autoregressive & 26.8{\scriptsize$\,\pm\,$4.6} & 37.5{\scriptsize$\,\pm\,$5.7} & 47.5{\scriptsize$\,\pm\,$6.3} & 56.3{\scriptsize$\,\pm\,$6.3} & \textbf{64.0}{\scriptsize$\,\pm\,$6.5} & \textbf{71.0}{\scriptsize$\,\pm\,$6.2} & \textbf{76.8}{\scriptsize$\,\pm\,$6.4} \\
 & Low Confidence ($T_{\mathrm{pos}}=0$) & \textbf{41.8}{\scriptsize$\,\pm\,$6.8} & \textbf{47.5}{\scriptsize$\,\pm\,$6.8} & \textbf{52.6}{\scriptsize$\,\pm\,$6.9} & \textbf{57.5}{\scriptsize$\,\pm\,$6.8} & 62.4{\scriptsize$\,\pm\,$6.8} & 67.2{\scriptsize$\,\pm\,$6.7} & 70.7{\scriptsize$\,\pm\,$7.0} \\
 & Random ($T_{\mathrm{pos}}=\infty$) & 11.5{\scriptsize$\,\pm\,$2.9} & 18.2{\scriptsize$\,\pm\,$4.0} & 26.3{\scriptsize$\,\pm\,$5.2} & 35.0{\scriptsize$\,\pm\,$5.9} & 43.9{\scriptsize$\,\pm\,$6.5} & 52.7{\scriptsize$\,\pm\,$6.8} & 61.6{\scriptsize$\,\pm\,$7.3} \\
 & \textbf{TLC} ($T_{\mathrm{pos}}=1$) & 25.9{\scriptsize$\,\pm\,$4.7} & 36.3{\scriptsize$\,\pm\,$5.7} & 46.3{\scriptsize$\,\pm\,$6.3} & 55.4{\scriptsize$\,\pm\,$6.3} & 63.7{\scriptsize$\,\pm\,$6.5} & 70.8{\scriptsize$\,\pm\,$6.4} & 76.2{\scriptsize$\,\pm\,$6.7} \\
\cmidrule(lr){1-9}
\multirow{4}{*}{MBPP} & Autoregressive & 30.3{\scriptsize$\,\pm\,$2.9} & 40.3{\scriptsize$\,\pm\,$3.4} & 49.4{\scriptsize$\,\pm\,$3.6} & \textbf{57.7}{\scriptsize$\,\pm\,$3.7} & \textbf{65.2}{\scriptsize$\,\pm\,$3.6} & \textbf{71.8}{\scriptsize$\,\pm\,$3.6} & \textbf{76.8}{\scriptsize$\,\pm\,$3.7} \\
 & Low Confidence ($T_{\mathrm{pos}}=0$) & \textbf{39.1}{\scriptsize$\,\pm\,$3.7} & \textbf{45.5}{\scriptsize$\,\pm\,$3.9} & \textbf{50.8}{\scriptsize$\,\pm\,$4.0} & 55.4{\scriptsize$\,\pm\,$3.9} & 59.5{\scriptsize$\,\pm\,$4.0} & 63.6{\scriptsize$\,\pm\,$4.0} & 67.4{\scriptsize$\,\pm\,$4.1} \\
 & Random ($T_{\mathrm{pos}}=\infty$) & 14.9{\scriptsize$\,\pm\,$1.9} & 23.1{\scriptsize$\,\pm\,$2.5} & 32.6{\scriptsize$\,\pm\,$3.1} & 42.5{\scriptsize$\,\pm\,$3.5} & 52.0{\scriptsize$\,\pm\,$3.8} & 60.3{\scriptsize$\,\pm\,$3.9} & 67.2{\scriptsize$\,\pm\,$4.1} \\
 & \textbf{TLC} ($T_{\mathrm{pos}}=1$) & 28.7{\scriptsize$\,\pm\,$2.7} & 39.4{\scriptsize$\,\pm\,$3.3} & 49.2{\scriptsize$\,\pm\,$3.6} & 57.7{\scriptsize$\,\pm\,$3.7} & 64.9{\scriptsize$\,\pm\,$3.6} & 71.1{\scriptsize$\,\pm\,$3.6} & 76.6{\scriptsize$\,\pm\,$3.7} \\
\midrule
\multicolumn{9}{c}{\textit{Dream-v0-Instruct-7B}} \\
\midrule
\multirow{4}{*}{GSM8K} & Autoregressive & 64.6{\scriptsize$\,\pm\,$3.3} & 79.2{\scriptsize$\,\pm\,$2.9} & 89.5{\scriptsize$\,\pm\,$2.2} & 95.5{\scriptsize$\,\pm\,$1.6} & \textbf{98.1}{\scriptsize$\,\pm\,$1.1} & \textbf{99.2}{\scriptsize$\,\pm\,$0.8} & \textbf{99.7}{\scriptsize$\,\pm\,$0.5} \\
 & Low Confidence ($T_{\mathrm{pos}}=0$) & \textbf{80.6}{\scriptsize$\,\pm\,$3.1} & \textbf{88.7}{\scriptsize$\,\pm\,$2.8} & \textbf{92.9}{\scriptsize$\,\pm\,$2.4} & 95.3{\scriptsize$\,\pm\,$2.1} & 96.6{\scriptsize$\,\pm\,$1.8} & 97.3{\scriptsize$\,\pm\,$1.8} & 97.7{\scriptsize$\,\pm\,$1.8} \\
 & Random ($T_{\mathrm{pos}}=\infty$) & 41.0{\scriptsize$\,\pm\,$2.5} & 60.6{\scriptsize$\,\pm\,$3.0} & 77.8{\scriptsize$\,\pm\,$3.0} & 88.8{\scriptsize$\,\pm\,$2.5} & 94.4{\scriptsize$\,\pm\,$2.0} & 97.0{\scriptsize$\,\pm\,$1.6} & 98.3{\scriptsize$\,\pm\,$1.5} \\
 & \textbf{TLC} ($T_{\mathrm{pos}}=1$) & 66.7{\scriptsize$\,\pm\,$2.8} & 82.9{\scriptsize$\,\pm\,$2.6} & 92.0{\scriptsize$\,\pm\,$2.1} & \textbf{96.2}{\scriptsize$\,\pm\,$1.7} & 97.9{\scriptsize$\,\pm\,$1.3} & 98.7{\scriptsize$\,\pm\,$1.1} & 99.3{\scriptsize$\,\pm\,$0.8} \\
\cmidrule(lr){1-9}
\multirow{4}{*}{MATH} & Autoregressive & 34.3{\scriptsize$\,\pm\,$2.9} & 45.9{\scriptsize$\,\pm\,$3.3} & 56.0{\scriptsize$\,\pm\,$3.6} & 64.7{\scriptsize$\,\pm\,$3.6} & 71.8{\scriptsize$\,\pm\,$3.4} & 77.3{\scriptsize$\,\pm\,$3.4} & 81.4{\scriptsize$\,\pm\,$3.4} \\
 & Low Confidence ($T_{\mathrm{pos}}=0$) & \textbf{42.3}{\scriptsize$\,\pm\,$3.3} & \textbf{52.2}{\scriptsize$\,\pm\,$3.5} & \textbf{60.6}{\scriptsize$\,\pm\,$3.7} & \textbf{67.9}{\scriptsize$\,\pm\,$3.6} & \textbf{73.8}{\scriptsize$\,\pm\,$3.4} & \textbf{78.3}{\scriptsize$\,\pm\,$3.3} & 81.8{\scriptsize$\,\pm\,$3.5} \\
 & Random ($T_{\mathrm{pos}}=\infty$) & 16.0{\scriptsize$\,\pm\,$1.8} & 25.6{\scriptsize$\,\pm\,$2.5} & 37.0{\scriptsize$\,\pm\,$3.1} & 48.6{\scriptsize$\,\pm\,$3.4} & 59.2{\scriptsize$\,\pm\,$3.5} & 68.1{\scriptsize$\,\pm\,$3.7} & 75.4{\scriptsize$\,\pm\,$3.8} \\
 & \textbf{TLC} ($T_{\mathrm{pos}}=1$) & 29.4{\scriptsize$\,\pm\,$2.7} & 41.1{\scriptsize$\,\pm\,$3.2} & 52.2{\scriptsize$\,\pm\,$3.5} & 61.9{\scriptsize$\,\pm\,$3.5} & 70.2{\scriptsize$\,\pm\,$3.4} & 77.0{\scriptsize$\,\pm\,$3.3} & \textbf{82.2}{\scriptsize$\,\pm\,$3.3} \\
\cmidrule(lr){1-9}
\multirow{4}{*}{HumanEval} & Autoregressive & 43.6{\scriptsize$\,\pm\,$4.6} & 59.5{\scriptsize$\,\pm\,$5.3} & \textbf{72.0}{\scriptsize$\,\pm\,$5.4} & \textbf{80.4}{\scriptsize$\,\pm\,$4.9} & \textbf{86.2}{\scriptsize$\,\pm\,$4.6} & \textbf{90.0}{\scriptsize$\,\pm\,$4.2} & \textbf{92.7}{\scriptsize$\,\pm\,$4.0} \\
 & Low Confidence ($T_{\mathrm{pos}}=0$) & \textbf{57.0}{\scriptsize$\,\pm\,$6.3} & \textbf{64.9}{\scriptsize$\,\pm\,$6.2} & 71.2{\scriptsize$\,\pm\,$6.0} & 76.3{\scriptsize$\,\pm\,$5.8} & 80.5{\scriptsize$\,\pm\,$5.5} & 84.0{\scriptsize$\,\pm\,$5.2} & 87.2{\scriptsize$\,\pm\,$5.2} \\
 & Random ($T_{\mathrm{pos}}=\infty$) & 12.4{\scriptsize$\,\pm\,$2.5} & 20.7{\scriptsize$\,\pm\,$3.8} & 31.4{\scriptsize$\,\pm\,$5.0} & 43.2{\scriptsize$\,\pm\,$5.7} & 55.0{\scriptsize$\,\pm\,$6.1} & 66.0{\scriptsize$\,\pm\,$6.4} & 74.4{\scriptsize$\,\pm\,$6.7} \\
 & \textbf{TLC} ($T_{\mathrm{pos}}=1$) & 32.8{\scriptsize$\,\pm\,$4.4} & 46.9{\scriptsize$\,\pm\,$5.4} & 60.0{\scriptsize$\,\pm\,$5.7} & 70.6{\scriptsize$\,\pm\,$5.6} & 78.6{\scriptsize$\,\pm\,$5.5} & 84.2{\scriptsize$\,\pm\,$5.1} & 87.2{\scriptsize$\,\pm\,$5.2} \\
\cmidrule(lr){1-9}
\multirow{4}{*}{MBPP} & Autoregressive & 46.8{\scriptsize$\,\pm\,$3.2} & 58.6{\scriptsize$\,\pm\,$3.4} & \textbf{67.5}{\scriptsize$\,\pm\,$3.4} & \textbf{74.3}{\scriptsize$\,\pm\,$3.3} & \textbf{79.9}{\scriptsize$\,\pm\,$3.1} & \textbf{84.4}{\scriptsize$\,\pm\,$3.0} & 87.6{\scriptsize$\,\pm\,$2.9} \\
 & Low Confidence ($T_{\mathrm{pos}}=0$) & \textbf{54.9}{\scriptsize$\,\pm\,$3.8} & \textbf{60.9}{\scriptsize$\,\pm\,$3.8} & 66.0{\scriptsize$\,\pm\,$3.7} & 70.4{\scriptsize$\,\pm\,$3.7} & 74.0{\scriptsize$\,\pm\,$3.5} & 77.3{\scriptsize$\,\pm\,$3.5} & 80.4{\scriptsize$\,\pm\,$3.6} \\
 & Random ($T_{\mathrm{pos}}=\infty$) & 21.8{\scriptsize$\,\pm\,$2.2} & 32.8{\scriptsize$\,\pm\,$2.8} & 44.7{\scriptsize$\,\pm\,$3.3} & 55.7{\scriptsize$\,\pm\,$3.5} & 64.8{\scriptsize$\,\pm\,$3.6} & 72.2{\scriptsize$\,\pm\,$3.5} & 78.2{\scriptsize$\,\pm\,$3.7} \\
 & \textbf{TLC} ($T_{\mathrm{pos}}=1$) & 40.6{\scriptsize$\,\pm\,$3.0} & 53.1{\scriptsize$\,\pm\,$3.3} & 63.6{\scriptsize$\,\pm\,$3.4} & 71.8{\scriptsize$\,\pm\,$3.3} & 78.4{\scriptsize$\,\pm\,$3.1} & 83.6{\scriptsize$\,\pm\,$3.0} & \textbf{87.8}{\scriptsize$\,\pm\,$2.9} \\
\bottomrule
\end{tabular}}
\end{table}

\newpage

\subsection{TCT (\Cref{fig:llada_tct_all_ds} \& \Cref{fig:dream_tct_all_ds}) }

\begin{table}[h]
\centering
\caption{Pass@$k$ (\%) for \textbf{LLaDA-8B-Instruct} with bootstrapped 95\% confidence intervals
(10{,}000 resamples over questions/tasks, $T_{\mathrm{token}}=0.8$, $\lambda=0.6$).
Values are mean $\pm$ half-width of the 95\% CI. Evaluation sets: GSM8K 300 questions, MATH 500, HumanEval 164 tasks, MBPP 500.
NFE is the average number of function evaluations per sample; for Autoregressive it is the effective generation length (tokens up to EOS), matching the pass@NFE axis of the figures.}
\label{tab:tct-passk-ci-llada}
\resizebox{\textwidth}{!}{%
\begin{tabular}{llrcccccccc}
\toprule
& & & \multicolumn{8}{c}{$k$} \\
\cmidrule(lr){4-11}
Dataset & Method & NFE & 1 & 2 & 4 & 8 & 16 & 32 & 64 & 96 \\
\midrule
\multirow{3}{*}{GSM8K} & Autoregressive & 228 & 75.8{\scriptsize$\,\pm\,$3.1} & 86.7{\scriptsize$\,\pm\,$2.7} & 92.6{\scriptsize$\,\pm\,$2.1} & 96.0{\scriptsize$\,\pm\,$1.6} & 97.8{\scriptsize$\,\pm\,$1.3} & 98.8{\scriptsize$\,\pm\,$1.0} & -- & -- \\
 & Fast-dLLM & 46 & 73.9{\scriptsize$\,\pm\,$3.1} & 85.5{\scriptsize$\,\pm\,$2.9} & 91.2{\scriptsize$\,\pm\,$2.5} & 94.1{\scriptsize$\,\pm\,$2.2} & 96.0{\scriptsize$\,\pm\,$1.9} & 97.1{\scriptsize$\,\pm\,$1.7} & 97.5{\scriptsize$\,\pm\,$1.7} & 97.7{\scriptsize$\,\pm\,$1.7} \\
 & TCT ($T_{\mathrm{pos}}=0.1$) & 44 & 65.7{\scriptsize$\,\pm\,$2.9} & 81.7{\scriptsize$\,\pm\,$2.9} & 90.2{\scriptsize$\,\pm\,$2.4} & 94.5{\scriptsize$\,\pm\,$2.0} & 96.7{\scriptsize$\,\pm\,$1.7} & 97.8{\scriptsize$\,\pm\,$1.5} & 98.2{\scriptsize$\,\pm\,$1.5} & 98.3{\scriptsize$\,\pm\,$1.5} \\
\cmidrule(lr){1-11}
\multirow{3}{*}{MATH} & Autoregressive & 243 & 29.6{\scriptsize$\,\pm\,$3.0} & 39.2{\scriptsize$\,\pm\,$3.4} & 48.1{\scriptsize$\,\pm\,$3.6} & 56.3{\scriptsize$\,\pm\,$3.7} & 64.0{\scriptsize$\,\pm\,$3.6} & 70.8{\scriptsize$\,\pm\,$3.5} & -- & -- \\
 & Fast-dLLM & 61 & 30.3{\scriptsize$\,\pm\,$3.1} & 39.0{\scriptsize$\,\pm\,$3.6} & 46.4{\scriptsize$\,\pm\,$3.8} & 53.4{\scriptsize$\,\pm\,$3.8} & 60.1{\scriptsize$\,\pm\,$3.8} & 66.0{\scriptsize$\,\pm\,$3.8} & 70.6{\scriptsize$\,\pm\,$3.9} & 72.6{\scriptsize$\,\pm\,$3.9} \\
 & TCT ($T_{\mathrm{pos}}=0.1$) & 58 & 25.3{\scriptsize$\,\pm\,$2.7} & 35.4{\scriptsize$\,\pm\,$3.3} & 44.5{\scriptsize$\,\pm\,$3.6} & 53.0{\scriptsize$\,\pm\,$3.7} & 60.8{\scriptsize$\,\pm\,$3.7} & 67.7{\scriptsize$\,\pm\,$3.7} & 73.5{\scriptsize$\,\pm\,$3.7} & 76.6{\scriptsize$\,\pm\,$3.7} \\
\cmidrule(lr){1-11}
\multirow{3}{*}{HumanEval} & Autoregressive & 241 & 26.8{\scriptsize$\,\pm\,$4.5} & 37.8{\scriptsize$\,\pm\,$5.7} & 47.9{\scriptsize$\,\pm\,$6.2} & 56.8{\scriptsize$\,\pm\,$6.3} & 64.6{\scriptsize$\,\pm\,$6.2} & 71.3{\scriptsize$\,\pm\,$6.1} & -- & -- \\
 & Fast-dLLM & 69 & 34.7{\scriptsize$\,\pm\,$6.2} & 41.3{\scriptsize$\,\pm\,$6.7} & 47.0{\scriptsize$\,\pm\,$6.9} & 52.3{\scriptsize$\,\pm\,$7.0} & 57.1{\scriptsize$\,\pm\,$7.0} & 61.1{\scriptsize$\,\pm\,$7.1} & 64.2{\scriptsize$\,\pm\,$7.2} & 65.9{\scriptsize$\,\pm\,$7.0} \\
 & TCT ($T_{\mathrm{pos}}=0.1$) & 62 & 24.0{\scriptsize$\,\pm\,$4.6} & 33.4{\scriptsize$\,\pm\,$5.7} & 42.2{\scriptsize$\,\pm\,$6.4} & 49.8{\scriptsize$\,\pm\,$6.7} & 56.4{\scriptsize$\,\pm\,$6.6} & 62.7{\scriptsize$\,\pm\,$6.6} & 69.2{\scriptsize$\,\pm\,$6.5} & 73.8{\scriptsize$\,\pm\,$6.7} \\
\cmidrule(lr){1-11}
\multirow{3}{*}{MBPP} & Autoregressive & 139 & 30.2{\scriptsize$\,\pm\,$2.9} & 40.2{\scriptsize$\,\pm\,$3.3} & 49.3{\scriptsize$\,\pm\,$3.5} & 57.6{\scriptsize$\,\pm\,$3.6} & 65.1{\scriptsize$\,\pm\,$3.6} & 71.6{\scriptsize$\,\pm\,$3.5} & -- & -- \\
 & Fast-dLLM & 46 & 34.9{\scriptsize$\,\pm\,$3.4} & 42.4{\scriptsize$\,\pm\,$3.8} & 48.4{\scriptsize$\,\pm\,$3.9} & 53.1{\scriptsize$\,\pm\,$4.0} & 57.3{\scriptsize$\,\pm\,$4.0} & 61.3{\scriptsize$\,\pm\,$4.0} & 65.0{\scriptsize$\,\pm\,$4.1} & 66.8{\scriptsize$\,\pm\,$4.1} \\
 & TCT ($T_{\mathrm{pos}}=0.1$) & 43 & 26.0{\scriptsize$\,\pm\,$2.6} & 36.5{\scriptsize$\,\pm\,$3.3} & 46.0{\scriptsize$\,\pm\,$3.6} & 53.9{\scriptsize$\,\pm\,$3.8} & 60.7{\scriptsize$\,\pm\,$3.8} & 66.8{\scriptsize$\,\pm\,$3.7} & 72.5{\scriptsize$\,\pm\,$3.7} & 75.4{\scriptsize$\,\pm\,$3.8} \\
\bottomrule
\end{tabular}}
\end{table}

\begin{table}[h]
\centering
\caption{Pass@$k$ (\%) for \textbf{Dream-v0-Instruct-7B} with bootstrapped 95\% confidence intervals
(10{,}000 resamples over questions/tasks, $T_{\mathrm{token}}=0.8$, $\lambda=0.6$).
Values are mean $\pm$ half-width of the 95\% CI. Fast-dLLM and TCT use 96 samples per question; Autoregressive uses 64--128 (dataset-dependent) and is reported to $k\le32$.
Evaluation sets: GSM8K 300 questions, MATH 500, HumanEval 164 tasks, MBPP 500.
NFE is the average number of function evaluations per sample; for Autoregressive it is the effective generation length (tokens up to EOS), matching the pass@NFE axis of the figures.}
\label{tab:tct-passk-ci-dream}
\resizebox{\textwidth}{!}{%
\begin{tabular}{llrcccccccc}
\toprule
& & & \multicolumn{8}{c}{$k$} \\
\cmidrule(lr){4-11}
Dataset & Method & NFE & 1 & 2 & 4 & 8 & 16 & 32 & 64 & 96 \\
\midrule
\multirow{3}{*}{GSM8K} & Autoregressive & 200 & 64.6{\scriptsize$\,\pm\,$3.3} & 79.1{\scriptsize$\,\pm\,$2.8} & 89.4{\scriptsize$\,\pm\,$2.2} & 95.2{\scriptsize$\,\pm\,$1.6} & 97.8{\scriptsize$\,\pm\,$1.1} & 99.0{\scriptsize$\,\pm\,$0.8} & -- & -- \\
 & Fast-dLLM & 55 & 69.0{\scriptsize$\,\pm\,$3.0} & 83.3{\scriptsize$\,\pm\,$2.8} & 90.9{\scriptsize$\,\pm\,$2.4} & 94.5{\scriptsize$\,\pm\,$2.1} & 96.3{\scriptsize$\,\pm\,$1.8} & 97.3{\scriptsize$\,\pm\,$1.7} & 97.8{\scriptsize$\,\pm\,$1.6} & 98.0{\scriptsize$\,\pm\,$1.5} \\
 & TCT ($T_{\mathrm{pos}}=0.1$) & 51 & 58.2{\scriptsize$\,\pm\,$2.7} & 77.1{\scriptsize$\,\pm\,$2.7} & 89.2{\scriptsize$\,\pm\,$2.2} & 95.0{\scriptsize$\,\pm\,$1.7} & 97.6{\scriptsize$\,\pm\,$1.3} & 98.6{\scriptsize$\,\pm\,$1.1} & 99.1{\scriptsize$\,\pm\,$1.0} & 99.3{\scriptsize$\,\pm\,$0.8} \\
\cmidrule(lr){1-11}
\multirow{3}{*}{MATH} & Autoregressive & 223 & 34.3{\scriptsize$\,\pm\,$2.9} & 45.9{\scriptsize$\,\pm\,$3.4} & 56.0{\scriptsize$\,\pm\,$3.6} & 64.7{\scriptsize$\,\pm\,$3.6} & 71.8{\scriptsize$\,\pm\,$3.5} & 77.3{\scriptsize$\,\pm\,$3.4} & -- & -- \\
 & Fast-dLLM & 66 & 36.4{\scriptsize$\,\pm\,$3.1} & 47.6{\scriptsize$\,\pm\,$3.5} & 57.2{\scriptsize$\,\pm\,$3.6} & 65.3{\scriptsize$\,\pm\,$3.6} & 71.9{\scriptsize$\,\pm\,$3.5} & 77.1{\scriptsize$\,\pm\,$3.4} & 80.7{\scriptsize$\,\pm\,$3.3} & 82.2{\scriptsize$\,\pm\,$3.3} \\
 & TCT ($T_{\mathrm{pos}}=0.1$) & 62 & 28.6{\scriptsize$\,\pm\,$2.5} & 41.0{\scriptsize$\,\pm\,$3.2} & 52.4{\scriptsize$\,\pm\,$3.5} & 61.9{\scriptsize$\,\pm\,$3.6} & 69.6{\scriptsize$\,\pm\,$3.5} & 75.7{\scriptsize$\,\pm\,$3.4} & 80.4{\scriptsize$\,\pm\,$3.3} & 82.6{\scriptsize$\,\pm\,$3.3} \\
\cmidrule(lr){1-11}
\multirow{3}{*}{HumanEval} & Autoregressive & 135 & 43.6{\scriptsize$\,\pm\,$4.6} & 59.6{\scriptsize$\,\pm\,$5.3} & 72.0{\scriptsize$\,\pm\,$5.3} & 80.4{\scriptsize$\,\pm\,$5.0} & 86.0{\scriptsize$\,\pm\,$4.6} & 89.5{\scriptsize$\,\pm\,$4.2} & -- & -- \\
 & Fast-dLLM & 42 & 44.4{\scriptsize$\,\pm\,$5.8} & 54.8{\scriptsize$\,\pm\,$6.1} & 63.6{\scriptsize$\,\pm\,$6.1} & 71.0{\scriptsize$\,\pm\,$5.9} & 76.7{\scriptsize$\,\pm\,$5.8} & 80.8{\scriptsize$\,\pm\,$5.6} & 84.0{\scriptsize$\,\pm\,$5.4} & 86.0{\scriptsize$\,\pm\,$5.2} \\
 & TCT ($T_{\mathrm{pos}}=0.1$) & 39 & 29.9{\scriptsize$\,\pm\,$4.2} & 43.4{\scriptsize$\,\pm\,$5.2} & 56.7{\scriptsize$\,\pm\,$5.7} & 67.8{\scriptsize$\,\pm\,$5.7} & 76.6{\scriptsize$\,\pm\,$5.3} & 83.3{\scriptsize$\,\pm\,$5.0} & 87.5{\scriptsize$\,\pm\,$4.8} & 89.0{\scriptsize$\,\pm\,$4.6} \\
\cmidrule(lr){1-11}
\multirow{3}{*}{MBPP} & Autoregressive & 69 & 46.8{\scriptsize$\,\pm\,$3.2} & 58.6{\scriptsize$\,\pm\,$3.4} & 67.5{\scriptsize$\,\pm\,$3.4} & 74.3{\scriptsize$\,\pm\,$3.2} & 79.9{\scriptsize$\,\pm\,$3.1} & 84.4{\scriptsize$\,\pm\,$2.9} & -- & -- \\
 & Fast-dLLM & 22 & 46.3{\scriptsize$\,\pm\,$3.6} & 54.0{\scriptsize$\,\pm\,$3.7} & 60.6{\scriptsize$\,\pm\,$3.7} & 66.5{\scriptsize$\,\pm\,$3.7} & 71.9{\scriptsize$\,\pm\,$3.5} & 76.5{\scriptsize$\,\pm\,$3.4} & 80.1{\scriptsize$\,\pm\,$3.3} & 81.8{\scriptsize$\,\pm\,$3.4} \\
 & TCT ($T_{\mathrm{pos}}=0.1$) & 23 & 33.5{\scriptsize$\,\pm\,$2.7} & 46.4{\scriptsize$\,\pm\,$3.3} & 57.5{\scriptsize$\,\pm\,$3.5} & 66.1{\scriptsize$\,\pm\,$3.5} & 72.9{\scriptsize$\,\pm\,$3.4} & 78.6{\scriptsize$\,\pm\,$3.2} & 83.4{\scriptsize$\,\pm\,$3.1} & 85.8{\scriptsize$\,\pm\,$3.1} \\
\bottomrule
\end{tabular}}
\end{table}

\end{document}